\documentclass[sigconf]{acmart}

\AtBeginDocument{%
  \providecommand\BibTeX{{%
    \normalfont B\kern-0.5em{\scshape i\kern-0.25em b}\kern-0.8em\TeX}}}

\newcommand{\filename}{ccs-template}

\def\compileFigures{0}
\newcounter{figureNumber}

\def\compileFprTprFigures{0}
\newcounter{fprTprFigureNumber}

\def\compileScatterFigures{0}
\newcounter{scatterFigureNumber}

\usepackage{amsmath,amsfonts,bm,amsthm}

\def\eqref#1{equation~\ref{#1}}
\def\1{\bm{1}}

\DeclareMathAlphabet{\mathsfit}{\encodingdefault}{\sfdefault}{m}{sl}
\SetMathAlphabet{\mathsfit}{bold}{\encodingdefault}{\sfdefault}{bx}{n}

\usepackage{tikz}
\usetikzlibrary{external}
\if\compileFigures1
\fi

\if\compileFprTprFigures1
\fi

\if\compileScatterFigures1
\fi

\usepackage{pgfplots}
\usepackage{hyperref}
\usepackage{url}
\usepackage{graphicx}
\pgfplotsset{compat=1.3}

\usepackage{multirow}
\usepackage{booktabs}

\usepackage[normalem]{ulem}

\usepackage{adjustbox}
\usepackage{tabularx}
\usepackage{subcaption}

\usepackage{enumitem}

\usepackage{caption}
\renewcommand\footnotetextcopyrightpermission[1]{} % removes footnote with conference information in first column
\begin{document}
	
\title{Enhanced Membership Inference Attacks against Machine Learning Models} 

\author{Jiayuan Ye}
\affiliation{%
  \institution{National University of Singapore}
  \country{}%Singapore}
  }
\email{jiayuan@comp.nus.edu.sg}

\author{Aadyaa Maddi}
\affiliation{%
  \institution{National University of Singapore}
  \country{}%Singapore}
}
\email{aadyaa.maddi@gmail.com}

\author{Sasi Kumar Murakonda}
\affiliation{%
 \institution{Privitar Labs}
 \country{}%UK}
}
\email{sasi.murakonda@privitar.com}

\author{Vincent Bindschaedler}
\affiliation{%
  \institution{University of Florida}
  \country{}%USA}
  }
\email{vbindsch@cise.ufl.edu}

\author{Reza Shokri}
\affiliation{
  \institution{National University of Singapore}
  \country{}%Singapore}
  }
\email{reza@comp.nus.edu.sg}

\begin{abstract}
How much does a machine learning algorithm leak about its training data, and why?  Membership inference attacks are used as an auditing tool to quantify this leakage.  In this paper, we present a comprehensive \textit{hypothesis testing framework} that enables us not only to formally express the prior work in a consistent way, but also to design new membership inference attacks that use reference models to achieve a significantly higher power (true positive rate) for any (false positive rate) error. More importantly, we explain \textit{why} different attacks perform differently. We present a template for indistinguishability games, and provide an interpretation of attack success rate across different instances of the game. We discuss various uncertainties of attackers that arise from the formulation of the problem, and show how our approach tries to minimize the attack uncertainty to the one bit secret about the presence or absence of a data point in the training set. We perform a \textit{differential analysis} between all types of attacks, explain the gap between them, and show what causes data points to be vulnerable to an attack (as the reasons vary due to different granularities of memorization, from overfitting to conditional memorization).  Our auditing framework is openly accessible as part of the \textit{Privacy Meter} software tool.~\footnote{\url{https://github.com/privacytrustlab/ml_privacy_meter/tree/master/research/2022_enhanced_mia}}

\end{abstract}

\maketitle
\pagestyle{plain} % removes running headers

% \input{sections/outline}

% !TEX root = ../main.tex

\section{Introduction}
\label{sec:intro}

Machine learning algorithms are under increasing scrutiny from regulatory authorities, due to their usage of large amount of sensitive data.  In particular, vulnerability to membership inference attacks~\citep{homer2008resolving, shokri2017membership} --- which allow determining if a specific data instance was part of the model's training set --- is highlighted as a potential confidentiality violation and privacy threat to training data by organizations such as the ICO (UK) and NIST (US)~\citep{murakonda2020ml}. Also, there is a tight connection between the underlying notion of privacy risk in membership inference and differential privacy~\citep{dwork2006calibrating, yeom2018privacy, jagielski2020auditing, nasr2021adversary}. Therefore, using membership inference analysis for privacy auditing and data protection impact assessment (DPIA) in machine learning systems is gaining traction. However, despite the plethora of membership inference attacks in the literature, and regardless of their reported empirical performance, the attacks and their associated risk measurement approaches are fundamentally limited.

\textit{Lack of consistency/comparability: }
The focus in the literature is in celebrating the empirically evaluated success of new attacks, even though different attacks could be incomparable as the meaning of their success rates is dependent on specific assumptions about the adversary's uncertainty about the target, instead of (or in addition to) the inherent leakage of the model.  As we show in Section~\ref{ssec:summary_attack_explain}, attacks that are seemingly similar (e.g., they all test the loss or output of model on data) can be actually measuring \textit{different} notions of leakage. Also, oversimplifying the notions of leakage to average-case and worst-case leakage, without formally specifying the random experiments in the membership indistinguishability games, can lead to more confusion. Thus, it is crucial to formalize attacks and their uncertainty in a consistent framework to develop meaningful metrics for their comparisons (Section~\ref{ssec:attack_comparison} and \ref{ssec:why_vulnerable}).

\textit{Lack of explanations and guidance: }
Prior work give little formal explanation of why some attack strategies succeed in inferring membership of training data and why other strategies fail.  Therefore the prior work does not provide any guidance on how to design stronger attacks, so every new work needs to embark on a new journey of exploration and discovery.  Besides, by implicitly assuming that there is only one right way to run membership inference attacks, the interpretations of attack results are inadequate: One analysis concludes that leakage is only due to overfitting~\cite{yeom2018privacy}, and another analysis concludes that the true leakage is rather due to average memorization of hard points~\cite{carlini2022membership}. We show there are many simultaneous reasons for leakage which are separately captured by different attacks (Section~\ref{ssec:why_vulnerable}), and can be expressed in a coherent formal hypothesis testing framework.

\renewcommand{\i}{\textbf{\textsf{in}} }
\renewcommand{\o}{\textbf{\textsf{out}} }

\textbf{What is the essential problem?} Assume we want to train a model using a machine learning algorithm. Given any data point, we can partition the set of possible worlds into two categories: the \i worlds where the data point is a member of the training set, and the \o worlds in which it is not.  Now, we form a hypothetical game with an adversary.  We design a random experiment that first selects a data point $x$, then \textit{samples} a pair of \i and \o worlds, and finally places the adversary in one of them at random. We then challenge the adversary to tell the world he is in, exclusively based on data point~$x$ and the trained model~$\theta$.  The adversary runs a membership inference attack to infer the secret bit \i or \textbf{\textsf{out}}. We repeat the random experiment many times, and quantify the adversary's power (as the true positive rate of his attack), and adversary's error (as the false positive rate of his attack), where positive is associated with the \i world. We can imagine that different adversaries might vary based on how much error they can tolerate, but they all aim at designing membership inference attacks that would maximize the adversary's power. We define the \textit{leakage} of the algorithm as the power versus error curve which is computed on the result of our random experiments on all the adversaries.

\textbf{How should we interpret the leakage of an algorithm?} It all depends on how the random experiment of the indistinguishability game is set up, and which sources of randomness are fixed over the experiments. For example, we can assume that the target data~$x$ is fixed, and the training datasets in the two worlds differ only in~$x$. Then, the attack results would reflect the \textit{average leakage} of the the training algorithm about~$x$, where averaging is done over all other data points, and the randomness of the training algorithm. In a more restricted setting, we can also assume that the remaining data points in the training set are also fixed (and decided by the adversary). The attack results, in this setting, reflect the \textit{conditional leakage} of the training algorithm about~$x$, conditioned on the given remaining data points, but averaged over the randomness of the algorithm. In a more generic setting, we can assume the training data of \i and \o worlds differ in one data point~$x$, but all the records including~$x$ are freshly sampled in each experiment. The attack results, in this setting, reflect the \textit{average leakage} of the training algorithm, averaged over all possible data points, and the randomness of the training algorithm. We can use the game template to construct all different types of experiments. However, we emphasize that, even though they all measure some type of leakage, the meaning of leakage is different in each type of random experiment. 

\textbf{Our core contributions are designing the strongest membership inference attacks against machine learning models, based on an explainable theoretical framework, and providing an accurate interpretation of their empirical results}. At the heart of our framework is a template for membership indistinguishability games, a binary hypothesis testing formulation of the problem, and attack methods based on likelihood ratio tests. Our framework al enables us to formally express the prior work, and compare them all in a consistent way. We compare attacks across the whole power (TPR) vs. error (FPR) curve, and highlight their performance at low FPR and high TPR regions. 

We formalize the shadow-based membership inference attacks~\cite{shokri2017membership}, and introduce an enhanced (population-based) attack that achieves a similar (or better) power for any given error, but with a \textit{significantly lower computation cost} (without requiring to train any shadow models). We also derive an advanced reference model-based attack that outperforms prior work in the same class (of using reference models)~\cite{carlini2022membership} in all metrics, as shown in Figure~\ref{fig:related_work_carlini}. Finally, we derive a novel attack based on self-distillation, which uses a model-dependent and sample-dependent strategy to further reduce the uncertainty of the attacker. By design, this attack tries to eliminate the attacker's uncertainty about the training data of the target model, which enables the attack to have a higher AUC score and a better performance (lower FPR) at a high TPR in Section~\ref{ssec:evaluation_attack_performance}. 

Our framework is an explanatory tool that allows consistent comparison between different designed attacks. For example, the widely-used method of computing the threshold from shadow models results in an attack that aims to optimize the success of a fixed attack in detecting that a randomly-selected data point is a member of the training set of a randomly-selected model. This is why the performance of such attacks --- which is largely sensitive to the signals applying to most members of most models --- has a strong relationship to overfitting (as shown in~\citep{shokri2017membership,yeom2018privacy}). 

Because our framework enables designing attacks of increasing strength, we perform a \textbf{differential analysis} and investigate why some data points are more vulnerable than others. Given some target models, we identify sets of data points that are vulnerable to some attacks but not to others (e.g., points detected by our reference model attack but not the population-based attack). Using these differentially vulnerable sets, we are able to identify the most vulnerable data records, in a much more precise way compared to the prior work. In Figure~\ref{fig:purchase100_hard_example_fpr_vs_tpr_plot}, we show that if we focus on such points, the AUC of the strongest attack is as high as $0.984$. We also explain the differential vulnerability through the loss landscape. For example, using reference-model attacks we find that the loss distribution of the top vulnerable points is more concentrated and has a larger mean than the overall loss distribution of the other data points. We also investigate how the vulnerability of some points is influenced by their nearest neighbors (defined in our case by the cosine distance of the output of the second to last layer). These experimental insights suggest new directions for the design of improved attack strategies.

% !TEX root = ../main.tex

\section{Related Work}
\label{sec:related-work}

\paragraph{Privacy Risk Analysis with Membership Inference Attacks}
\cite{homer2008resolving} performed the first membership inference attack on genome data to identify the presence of an individual's genome in a mixture of genomes. \cite{sankararaman2009genomic, backes2016membership} provided a formal analysis of this risk of detecting the presence of an individual from aggregate statistics computed on independent (binary or continuous) attributes. \cite{murakonda2021quantifying} extended this analysis to the case of releasing discrete Bayesian networks learned from data with dependent attributes. \cite{dwork2015robust} provide a more extensive analysis when the released statistics are noisy and the attacker has only one reference sample to perform the attack. These works establish the privacy risk of releasing aggregate statistics by quantifying the success of membership inference attacks as a function of the number of statistics released and individuals in the dataset. We refer the reader to \cite{dwork2017exposed} for a survey. 

\paragraph{Differential Privacy and Membership Inference:}
The definitions of differential privacy~\citep{dwork2006calibrating} and membership inference~\citep{homer2008resolving,dwork2015robust,shokri2017membership} are very closely connected. 
By definition, differentially private algorithms bound the success of membership inference attacks for distinguishing between two neighboring datasets. Multiple works~\citep{yeom2018privacy,erlingsson2019we,humphries2020differentially,thudi2022bounding}, each improving on the previous work, have provided upper bounds on the \textit{average} success of membership inference attacks over \textit{general} targets as a function of the parameters in differential privacy. \cite{jayaraman2019evaluating,rahman2018membership} evaluated the performance of membership inference attacks on machine learning models trained with differentially private algorithms. Moreover, the empirical performance of membership inference attacks has also been used to provide lower bounds on the privacy guarantees achieved by various differentially private algorithms~\citep{jagielski2020auditing,nasr2021adversary,malek2021antipodes}. These works call for stronger membership inference attacks that could measure the leakage through the model about a particular point of interest, which is what we are trying to do in this paper.

\paragraph{Membership Inference Attack for Machine Learning Models}
\cite{shokri2017membership} demonstrated the vulnerability of machine learning models to membership inference attacks in the black-box setting, where the adversary has only query access to the target model. The attack algorithm is based on the concept of shadow models, which are models trained on some attacker dataset that is similar to that of the training data. A substantial body of literature followed this work extending the shadow model attacks to different setting such as white box analysis~\citep{nasr2019comprehensive, sablayrolles2019white, leino2020stolen}, label-only access~\citep{li2021membership,choquette2021label}, and federated learning~\citep{nasr2019comprehensive, melis2019exploiting}. However,
for construction and evaluation of the attack strategy, most previous works follow the membership inference game formulation in \cite{yeom2018privacy}, which focus attacking \textit{average} member and non-member data records of a target model. 
Such an average-case formulation does not support reasoning about the privacy risk of individual data records. As a result, the attacks optimized under this framework, despite chasing higher and higher \textit{average accuracy}, are still failing to precisely capture the privacy risk of worst-case data records in DP auditing~\cite{jagielski2020auditing,nasr2021adversary}. On the contrary, we formalize new membership inference attack by reasoning about the \textit{worst-case performance} of attacks over specific targets, thus identifying records with higher vulnerabilities (Figure~\ref{fig:purchase100_hard_example_fpr_vs_tpr_plot}).

\paragraph{Beyond Average Performance of Membership Inference Attacks.} To understand the privacy risk on \textit{worst-case} training data, many recent works restrict the targets to a (heuristically-selected) subset of records~\cite{long2020pragmatic,carlini2019secret} or poisoned neighboring datasets~\cite{jagielski2020auditing,nasr2021adversary,tramer2022truth}, and observe significantly higher attack success rates. Partially motivated by these observations, and by prior works~\cite{murakonda2021quantifying} that computes likelihood ratio to launch powerful membership inference attacks, recent works~\cite{carlini2022membership,watson2021importance} design per-example attacks that perform \textit{better} with \textit{higher confidence} (e.g. \citet{carlini2022membership} argue the importance of low FPR), by studying the ``hardness'' of each data record and performing (parametric) difficulty calibration. These attacks are very similar in design to the Attack R in our paper, in the sense that each target data is considered separately (rather than aggregated in average) in constructing the attack. However, these works offer limited formal understanding about: what are the intrinsic properties of these worst-case ``hard'' records, and why are their proposed attacks better at identifying them. By contrast, we investigate \textit{why} different attacks identify different vulnerable examples, and how to \textit{understand} the vulnerabilities of records identified by attacks with different reduced uncertainties. We also show that our systematic efforts result in attacks that are more powerful than~\cite{carlini2022membership} (higher AUC scores with comparable TPR at small FPR) in Section~\ref{ssec:comparison_with_carlini}.

\section{Attack Framework}

\label{sec:framework}

Our objective is to design a framework that enables auditing the privacy loss of a machine learning models about \textit{a particular record}, in the \textit{black-box} setting (where only model outputs ---and not their parameters or internal computations--- are observable). This framework has three elements: (i) the inference game as the evaluation setup; (ii) the indistinguishability metric to measure the privacy risk, and (iii) the construction of membership inference attack as hypothesis testing. The notion of privacy underlying our framework is primarily based on differential privacy, and multiple pieces of this framework are generalizations of existing inference attacks against machine learning algorithms. We present the important design choices for the game while constructing and evaluating membership inference attacks, for the purpose of having more precise privacy auditing for different kinds of privacy loss. 

\subsection{Inference Game}

We quantify different types of privacy loss for training machine learning models in multiple hypothetical \textbf{inference games} between a challenger and an adversary. We first introduce a most general inference game that captures the average privacy loss of random models (trained on random subsets of a population data pool) about its (whole) training dataset, as follows.
\begin{definition}[Membership inference game for average model and record]
    \label{def:game_random}
    Let $\pi$ be the underlying pool of population data, and let $\mathcal{T}$ be the training algorithm of interest. The game between a challenger and an adversary proceeds as follows.
    \begin{enumerate}
    \item The challenger samples a dataset $D\xleftarrow{s_{D}}\pi^n$ using a \textbf{fresh} random seed $s_{D}$, and trains a model $\theta\xleftarrow{s_{\theta}}\mathcal{T}(D)$ on $D$ by using a \textbf{fresh} random seed $s_{\theta}$ in the algorithm $\mathcal{T}$.
    \item The challenger samples a data record $z_0\xleftarrow{s_{z_0}}\pi$ using a \textbf{fresh} random seed $s_{z_0}$. Note that here $z_0\notin D$ with high probability when the population data pool is large enough.
    \item The challenger samples a data record $z_1\xleftarrow{s_{z_1}} D$ using a \textbf{fresh} random seed $s_{z_1}$.
    \item The challenger flips a random unbiased coin $b\xleftarrow{R}\{0,1\}$, and sends the target model and target record $\theta,z_b$ to the adversary.
    \item The adversary gets access to the data population pool $\pi$ and access to the target model, and outputs a bit $\hat{b}\leftarrow \mathcal{A}(\theta,z_b)$.
    \item If $\hat{b}=b$, output $1$ (success). Otherwise, output $0$.
\end{enumerate}
\end{definition}

We compute the performance of the attack, by averaging it over many repetitions of this random experiment. This game is similar to the games in prior works~\citet{sablayrolles2019white,yeom2018privacy,carlini2022membership}, in the sense that both the target model and the target record are randomly generated. However, this also limits the type of leakage that this game captures, as it is averaged over multiple target models and data records. In the rest of this section, we introduce different variants of this game (in Definition~\ref{def:game_model}, \ref{def:game_record}, \ref{def:game_worstcase}) that capture the privacy loss of (a specific) target model about (a fixed) target data record.

\begin{definition}[Membership inference game for a \textbf{fixed} model]
    \label{def:game_model}
    \ 
    \begin{enumerate}
        \item The challenger samples a dataset $D\xleftarrow{s_{D}}\pi^n$ via a \textbf{fixed} random seed $s_{D}$, and trains a model $\theta\xleftarrow{s_{\theta}}\mathcal{T}(D)$ on the $D$ by using a \textbf{fixed} random seed $s_{\theta}$ in the algorithm $\mathcal{T}$.
        \item The challenger samples a data record $z_0\xleftarrow{s_{z_0}}\pi$ via a \textbf{fresh} random seed $s_{z_0}$. Note that here $z_0\notin D$ with high probability when the population data pool is large enough.
        \item The challenger samples a data record $z_1\xleftarrow{s_{z_1}} D$ via a \textbf{fresh} random seed $s_{z_1}$.
        \item Remaining steps are the same as (4) to (6) in Definition~\ref{def:game_random}.
    \end{enumerate}
\end{definition}

Note that this game is similar to the game for average model and record (Definition~\ref{def:game_random}), except that in step (1) the random seeds $s_{D}, s_{\theta}$ are \textbf{fixed}, such that the challenger always selects the same target dataset and target model across multiple trials of the game. Therefore, Definition~\ref{def:game_model} quantifies the privacy loss of a specific model trained on a \textbf{fixed} dataset. Similar games are widely used in practical MIA evaluations~\citep{shokri2017membership,nasr2019comprehensive,watson2021importance,salem2019ml} for auditing the privacy loss of a released model in machine-learning-as-a-service setting.

\begin{definition}[Membership inference game for a \textbf{fixed} record]\
    \label{def:game_record}
    \ 
    \begin{enumerate}
        \item The challenger samples a dataset $D\xleftarrow{s_{D}}\pi^n$ via a \textbf{fresh} random seed $s_{D}$, and trains a model $\theta_0\xleftarrow{s_{\theta_0}}\mathcal{T}(D)$ on $D$ by using a \textbf{fresh} random seed $s_{\theta_0}$ in the algorithm $\mathcal{T}$.
        \item The challenger samples a data record $z\xleftarrow{s_{z}}\pi$ via a \textbf{fixed} random seed $s_{z}$. Note that here $z\notin D$ with high probability when the population data pool is large enough.
        \item The challenger trains a model $\theta_1\xleftarrow{s_{\theta_1}} \mathcal{T}(D \cup \{z\})$ by using a \textbf{fresh} random seed $s_{\theta_1}$ in the algorithm $\mathcal{T}$.
        \item The challenger flips a random unbiased coin $b\xleftarrow{R}\{0,1\}$, and sends the target model and target record $\theta_b,z$ to the adversary.
    \item The adversary gets access to the data population pool $\pi$ and access to the target model, and outputs a bit $\hat{b}\leftarrow \mathcal{A}(\theta_b,z)$.
    \item If $\hat{b}=b$, output $1$ (success). Otherwise, output $0$.
    \end{enumerate}
\end{definition}

There are two differences between this game and the Definition~\ref{def:game_random} game. Firstly, the construction of target model and target record in step (3) is different. Secondly, in step (2) the random seed $s_{z}$ is \textbf{fixed} such that the challenger always selects the same target record across multiple trials of the game. In essence, the adversary is distinguishing between models trained with and without a particular record (while the remaining dataset $D$ is randomly sampled from population), and therefore tries to exploit the privacy loss of a \textbf{fixed} specific record. Similar inference games that only target (a) specific record(s) are used in previous works for pragmatic, high-precision membership inference attacks on a subset of vulnerable records~\cite{long2018understanding,long2020pragmatic}, and for estimating privacy risks of data records in different subgroups~\cite{chang2021privacy}.

\begin{definition}[Membership Inference Game for \textbf{fixed} worst-case record and dataset]
    \label{def:game_worstcase}
    \ 
    \begin{enumerate}
        \item The challenger samples a dataset $D\xleftarrow{s_{D}}\pi^n$ via a \textbf{fixed} random seed $s_{D}$, and trains a model $\theta_0\xleftarrow{s_{\theta_0}}\mathcal{T}(D)$ on $D$ by using a \textbf{fresh} random seed $s_{\theta_0}$ in algorithm $\mathcal{T}$.
        \item The challenger samples a data record $z\xleftarrow{s_{z}}\pi$ via a \textbf{fixed} random seed $s_{z}$.
        \item The challenger trains a model $\theta_1\xleftarrow{s_{\theta_1}} D\cup \{z\}$ by using a \textbf{fresh} random seed $s_{\theta_1}$ in algorithm $\mathcal{T}$.
        \item Remaining steps are the same as (4) to (6) in Definition~\ref{def:game_record}.
    \end{enumerate}
\end{definition}

Note that this game is similar to the game in Definition~\ref{def:game_record} for a \textbf{fixed} record except that in step (1) the random seed $s_{D}$ is also \textbf{fixed} such that the challenger always selects the same target dataset across multiple trials of the game. Therefore, the game Definition~\ref{def:game_worstcase} quantifies the privacy loss of a specific record with regard to a \textbf{fixed} dataset. When the record and dataset are fixed to be worst-case (crafted), this game closely resembles the type of worst-case leakage in differential privacy definition. Therefore, this game (Definition~\ref{def:game_worstcase}) is widely used for auditing differentially private learning algorithms~\cite{jagielski2020auditing,nasr2021adversary,tramer2022debugging}. 

\subsection{Indistinguishability Metric}

We use an \textbf{indistinguishability} measure (which is the basis of differential privacy) to define privacy of individual training data of a model. That is, the true leakage of a model $\theta$ (trained on private dataset $D$) about a target data $z_b$, is the degree to which a new model $\theta'$ trained on the same dataset excluding $z_b$ is distinguishable from $\theta$.  According to this measure, the \textit{privacy loss} of the model with respect to its training data is the adversary's success in distinguishing between the two possibilities ${b = 0} \text{ vs } {b = 1}$ over multiple repetitions of the inference game.  Naturally, the inference attack is a \textbf{hypothesis test}, and adversary's error is composed of the false positive (i.e., inferring a non-member as member) and false negative of the test.  In practice, the error of the adversary in each round of the inference game depends on multiple factors:

\begin{itemize}[leftmargin=1.25em]
	\item The \textit{true leakage of the model} (trained on a private dataset via a training algorithm) about the target data $z_b$ when $b=1$, i.e. when the private dataset contains $z_b$. As the true leakage becomes higher (or lower), the attack success rate on the corresponding model and data would increase (or decrease).
		
	\item The \textit{uncertainty} (belief or background knowledge) of attack algorithm \textit{about the population data} (where the target model's training data is sampled from). The more precise this background knowledge is, the higher the attacker's success rate will be. Note that this has nothing to do with the leakage due to releasing trained model, yet it is measured by attacks.
	
	\item The adversary's \textit{uncertainty about the training algorithm}~$\mathcal{T}$. For example an attack designed against training with cross-entropy loss may perform poorly on models trained with 0-1 loss.
	
	\item The \textit{dependency of the attack construction process} on the \textit{specific target} data $(x_b, y_b)$, and target model $\theta$. If an attacker uses the same strategy across all targets, it is very likely to fail on atypical targets, thus having suboptimal performance. This is because the attack ignores individual properties of targets.
	
	\item The attacker's \textit{uncertainty about all training data} (for a target model), except whether the target data $(x_b, y_b)$ is used. In the most extreme case, an adversary that exactly knows all the remaining training records, could perform the leave-one-out attack, which is in principle the most powerful attack (Section~\ref{subsec:loo_explain}).
	
\end{itemize}

\newcommand{\txtbullet}{\quad \textbullet\,\,}

In the ideal setting, we want the attack error  to be only dependent on the true leakage of the model about the target data (i.e., whether the same model trained with and without $(x_b, y_b)$ are distinguishable from each other).  To measure this leakage accurately, and to cancel out the effect of other uncertainties and factors, the evaluation setup for the inference game need to be designed based on the following principle: the population data used for constructing the attack algorithm, and evaluating the inference game, need to be similar, in distribution, to the training data. This is to minimize the impact of prior belief about target data in the performance of the inference attack. 
If this principle is violated, we might overestimate the privacy loss (by making the attack accuracy dependent on a distinct prior knowledge, e.g. all members are red and all non-members are blue) or underestimate the privacy loss (by evaluating the inference attack on a population data distribution for which it was not constructed).

Another crucial requirement is that the privacy audit needs to output a detailed report, which captures the uncertainty of the attack. Reporting just the accuracy of the attack, as in most of the literature, is not an informative report (as it does not necessarily capture the true leakage). Given the attack being a hypothesis test, the audit report needs to include the analysis of the error versus power of the test: if we can tolerate a certain level of false positive rate in the inference attack, how much would be the true positive rate of the attack, over the random samples from member and non-member data? The area under the curve for such an analysis reflects the chance that the membership of a random data point from the population or the training set can be inferred correctly.

\section{Constructing Black-box Membership Inference Attacks}
\label{section3}

In the membership inference game, the adversary can observe the output of a target machine learning model~$\theta$ (trained on unknown private dataset~$D$) and a precise target data point $z$ as input, and is expected to output $0$~or~$1$ to guess whether the sample $z$ is in the dataset $D$ or not. 
In this section, we construct and evaluate attacks under the following assumptions.
\begin{enumerate}[leftmargin=2em]
    \item \textit{Black-box access:} the adversary could only access the output function of the target model instead of white-box parameters.
    \item \textit{Knowledge of population data:} the adversary can sample from the underlying pool of population data. 
    \item \textit{Knowledge of training algorithm:} the adversary can sample from the distribution of trained model $\theta$ on any input dataset $D$ via the given randomized training algorithm $\mathcal{T}$. 
\end{enumerate}

In terms of signal function for the black-box membership inference attack, in this paper, we solely consider loss values. This is because, under binary hypothesis test formulation of membership inference attack game on \textit{general} targets (Definition~\ref{def:game_random}), we observe that the most powerful criterion for choosing among hypotheses -- the likelihood ratio test (LRT), is \textit{approximately} comparing the loss of target to a constant threshold (under certain assumptions and approximation steps).

\begin{lemma}[Approximated LRT for Membership Inference on General Targets]
    \label{lem:LRT}
    Let $(\theta, z)$ be random samples from the joint distribution of target model and target data point, specified by one of the following membership hypotheses.
    \begin{align}
        H_0:\ & D\xleftarrow{n\ i.i.d. samples}\pi, \theta\xleftarrow{sample}\mathcal{T}(D), z\xleftarrow{sample} \pi\label{eqn:LRT_hyposis}\\
        H_1:\ & D\xleftarrow{n\ i.i.d. samples}\pi, \theta\xleftarrow{sample}\mathcal{T}(D), z\xleftarrow{sample} D\label{eqn:alternative_hyposis}
    \end{align}
    Then under certain assumptions (in Appendix A) the Likelihood Ratio Test (LRT) \textbf{approximately} equals
    \begin{equation}
        \label{eqn:lrt_strategy}
        \text{If } \ell(\theta,x_z,y_z)\leq c, \text{ reject }H_0,
    \end{equation}
    where $c$ could be an arbitrarily constant threshold. 
\end{lemma}

We provide a detailed proof in Appendix A. It should be noted that Lemma~\ref*{lem:LRT} is \textit{not} an optimality guarantee, because the assumptions may not hold in practical settings, and the approximations are only accurate when the dataset is large enough. Instead, Lemma~\ref{lem:LRT} is an explanation of why (under certain assumptions and approximations) the loss-threshold based attacks are powerful in principle, which is also the reason that we focus on loss-based attacks.~\footnote{This is consistent with~\citet{sablayrolles2019white}, which proves that loss-base attack coarsely approximates the Bayes-optimal strategy under similar assumptions.}

As discussed in Section~\ref{sec:framework}, attack success on general targets is \textit{not enough} for auditing more specific types of privacy loss, where we actually want attacks that succeed on a \textit{specific} target model or a \textit{worst-case} data record. To this end, the general attack \eqref{eqn:lrt_strategy} with constant threshold $c$ is \textit{overly general} and may have suboptimal performance, as it fails to capture the individual patterns for \textit{atypical targets}. This motivates us to design attacks with model-dependent and record-dependent thresholds, such that the thresholds are optimized under  more specific target-dependent inference games (Definition~\ref{def:game_model}, \ref{def:game_record} or~\ref{def:game_worstcase}). These target-dependent thresholds and inference games model \textit{reduced uncertainty} of the attacker about its target (which is reflected by the fixed random seeds), and therefore conceptually improve attack performance.

\paragraph{Our General Template for Attack Strategy.} Building on the approximated general LRT \eqref{eqn:lrt_strategy}, we derive the following new variant of instance-specific attack strategy.
\begin{equation}
    \label{eqn:general_strategy}
    \text{If } \ell(\theta,x_z,y_z)\leq c_{\alpha}(\theta,x_z,y_z), \text{ reject }H_0.
\end{equation}
Here $c_{\alpha}(\theta,x_z,y_z)$ is a threshold function (to be computed) that depends on the target model and (or) data record, and satisfies an arbitrary confidence requirement $0\leq \alpha\leq 1$. In this new strategy, the threshold function may take different values on different pairs of target model and data, and therefore gives the adversary more freedom to optimize its threshold for the more specific inference games Definition~\ref{def:game_model} (of a fixed model), Definition~\ref{def:game_record} (of a fixed record) or Definition~\ref{def:game_worstcase} (of a fixed pair of record and dataset). 

\begin{table}[h!]
    \centering
    \caption{Summary of threshold function used in Attacks S, P, R, D, their dependency on the target and corresponding inference games.}
    {\footnotesize\begin{tabular}{|c|c|c|c|}
    \hline
    \textbf{Attack} & \textbf{Threshold} & \textbf{Dependencies} & \textbf{Inference Game} \\
    \hline
    S & $c_{\alpha}(y_z)$ & Record Label & Definition~\ref{def:game_random} \\
    \hline
    P & $c_{\alpha}(\theta)$ & Model (dataset) & Definition~\ref{def:game_model}\\
    \hline
    R & $c_{\alpha}(x_z, y_z)$ & Record & Definition~\ref{def:game_record}\\
    \hline
    D & $c_{\alpha}(\theta, x_z, y_z)$ & Record and Model & Definition~\ref{def:game_worstcase} \\
    \hline
    \end{tabular}}
    \label{tab:attack_threshold_dependency_summary}
\end{table}

We construct a series of attacks whose thresholds have increasingly strong dependency on the target model and (or) target record, in Table~\ref{tab:attack_threshold_dependency_summary}. The stronger the dependency is, the smaller the out world (that contains targets generated by the challenger under $b=0$ in the inference game) is, thus reflecting reduced uncertainty of the adversary about its target and enhancing attack performance.
Among the four attacks, Attack S optimizes a fixed threshold over the largest out world (inference game~\ref{def:game_random} under $b=0$), and recovers the shadow model attack~\cite{shokri2017membership}. Attack R optimizes its threshold over a smaller target record-dependent out world (inference game~\ref{def:game_record} under $b=0$), and is similar in nature to previous attacks~\cite{long2020pragmatic,watson2021importance,carlini2022membership} against a subset of vulnerable records. We also construct a most specific Attack D, which optimizes its threshold over approximation of the smallest out world (inference game~\ref{def:game_worstcase} under $b=0$) that only contains a set of models trained on the same training data as the target model, excluding the specific target data. 

\paragraph{Our General Method for Determining Attack Threshold} Our goal is to compute an attack threshold that \textbf{(1)} has a specified \textit{dependency} on the target (model and data), i.e., encompasses a given uncertainty level; \textbf{(2)} enables \textit{confident} attack prediction that guarantees \textit{low false positive rate (FPR)}. To achieve these goals, we compute the threshold under an arbitrary FPR requirement $0\leq \alpha\leq 1$, by deriving the percentiles of distribution over loss value (or other signal function values) in each out world. 
\begin{equation}
    \label{eqn:obj_general_alpha}
    \mathbb{E}_{P_{out}(D, \theta, z)}\left[\mathbf{1}_{\ell(\theta,x_z,y_z)\leq c_{\alpha}(\theta,x_z,y_z)}\right]= \alpha
\end{equation}
To solve \eqref{eqn:obj_general_alpha}, we first compute empirical loss histogram in the out world $P_{out}$ of the inference game that corresponds to give threshold dependence. We then use various smoothing methods (with details in Appendix B.3) to estimate the CDF for the loss distribution, and then compute its $\alpha$-percentile for any $0\leq \alpha\leq 1$. 

In this way, we compute threshold that approximately \textit{guarantees} FPR $\alpha$ \textit{before} seeing any actual evaluation data for the attack. This is significantly stronger than prior attacks, because \textit{prior works} do \textit{not} aim to (or cannot) construct attacks (as a generic rejection rule in the hypothesis test) for any given FPR. More specifically, \citet{shokri2017membership}, and all the follow-up works based on shadow training, construct a “single attack” that tries to heuristically optimize TPR+TNR. \citet{sablayrolles2019white}, does the same in a more rigorous way (by finding the Bayes optimal solution). \citet{yeom2018privacy}, and its follow-up works, and \citet{carlini2022membership}, do not construct an optimal generic attack, but rather construct a rule for membership detection. However, they only empirically evaluate both FPR and TPR by sweeping through all possible values for the threshold in their rule. A real attacker however does not sweep thresholds. Since the threshold does not have a simple relation to FPR, their attacks do not offer a way to decide which attack threshold to use in order to satisfy a given FPR requirement (before the evaluation).

\subsection{Attack S: MIA via \underline{S}hadow models}
Starting from ~\citet{shokri2017membership}, a substantial body of literature~\citet{nasr2019comprehensive, sablayrolles2019white, leino2020stolen,long2018understanding,song2021systematic,salem2019ml} studied and improved the shadow model attack methodology. All these works follow the same attack framework for membership inference, but they either use a different attack statistics (such as loss or confidence score), or find a more efficient way to perform the attacks. 

 Therefore, we first formalize a strong baseline shadow model membership inference Attack S based on ~\citet{shokri2017membership},~that effectively uses label-dependent attack threshold $c_{\alpha}(y_z)$, as follows.
\begin{equation}
    \text{If } \ell(\theta,x_z,y_z)\leq c_{\alpha}(y_z), \text{ reject }H_0,\nonumber
\end{equation}
Since the threshold function $c_{\alpha}(y_z)$ is constant for all model $\theta$ and data feature $x_z$ (given fixed $y_z=y_0$), this strategy corresponds to the most general form of inference game (Definition~\ref{def:game_random}). Therefore, each out world (i.e., datasets, models and records generated under $b=0$ in the inference game) in the game is as follows. 

\begin{align}
    P_{out}(D,\theta,z):\ & D\xleftarrow{n\ i.i.d.samples} \pi,\nonumber\\
    & \theta\leftarrow\mathcal{T}(D), z\xleftarrow{sample}\pi_{y_z = y_0}. \label{eqn:S_pout}
\end{align}

Here $\pi_{y_z = y_0}$ is the conditional distribution of population data given fixed label $y_0$. The above   out world only reduces the attacker's uncertainty with regard to the label of the target data.

In practice, the attacker could approximate the model-record pair distribution $P_{out}$ in the induced out world (in \eqref{eqn:S_pout}), with the empirical distribution over each of the following sets $S_y$ of shadow models and shadow data points ($y$ could be any specified label).
\begin{align}
    S^y=&\cup_{i=1,2,\cdots} \{(\theta_i,z^i_1), (\theta_i,z^i_2),\cdots\}\nonumber\\
    &where\ \theta_i\leftarrow\mathcal{T}(D_i),\ D_i\xleftarrow{n\ i.i.d. samples}pop\nonumber\\
    & z_{1}^i,z_2^i,\cdots \xleftarrow{i.i.d.} \pi_{y_z=y_0}\nonumber
\end{align}
By the low false positive rate requirement, at most $\alpha$ fraction of these non-member instances (i.e., shadow models and population points) incur loss values smaller than the threshold function, i.e.,
\begin{equation}
    \label{eqn:approx_obj_shadow_alpha}
    \frac{ \lvert
    \left\{
        (\theta,z)\in S^{y_0}:\ell(\theta,x_z,y_z)\leq c_{\alpha}(y_0)
    \right\}
    \rvert}{|S^{y_0}|} = \alpha.
\end{equation}
The solution to \eqref{eqn:approx_obj_shadow_alpha} recovers the class-dependent threshold $c_{\alpha}(y_z)$ that we use in Attack S. Observe that this threshold has no dependency on the target feature $x_z$ or the target model $\theta$. Therefore, Attack S only reduces the uncertainty with regard to label, but does not reduce any uncertainty that the attacker has regarding the target data and target model. This can be seen from the threshold values (or shapes of loss histograms) of Attack S for different targets in Figure~\ref{fig:purchase100_2a_all_attack_loss_dist_plots_new}, which are the same across all four subplots.

\subsection{Attack P: model-dependent MIA via \underline{P}opulation data}

Can we design an attack with better performance by exploiting the dependence of loss threshold on different models? In this section, we design a new model-dependent membership inference Attack P that applies different attack threshold $c_{\alpha}(\theta)$ for different target model~$\theta$. The rationale for this design is to construct an inference attack which exploits the similar statistics as in Attack S, in a more accurate way by computing it only on the target model (instead of on all shadow models), yet with fewer computations (without the need to train shadow models). Similar techniques utilizing population data to infer membership are used in Genomics~\cite{homer2008resolving}, however, to the best of our knowledge, these techniques is not previously used in MIA for machine learning. The hypothesis test with model-dependent attack threshold is as follows.
\begin{equation}
    \text{If } \ell(\theta,x_z,y_z)\leq c_{\alpha}(\theta), \text{ reject }H_0,\nonumber
\end{equation}
Since the threshold function $c_{\alpha}(\theta)$ is constant for all target data feature $x_z$ and label $y_z$ (given fixed target model $\theta$), this strategy corresponds to the inference game for a fixed model (Definition~\ref{def:game_model}). Each out world (i.e., datasets, models and records generated under $b=0$ in the inference game) is as follows.
\begin{align}
    P_{out}(D, \theta, z) :\ & D = D_0, \theta = \theta_0, z\xleftarrow{sample} \pi
    \label{eqn:pout_P}
\end{align}

In each out world, the model is fixed to be a given target model $\theta_0$ (trained on private dataset $D_0$), while the target record is randomly sampled from the population data pool. Therefore, the out world reduces the attacker's uncertainty about a specific target model $\theta_0$. 

Empirically, the attacker could approximate this out world with empirical distribution over the following set $P^{\theta_0}$ of records sampled from the data population pool ($\theta_0$ could be any fixed model). 
\begin{align}
    P^{\theta_0}=\{(\theta_0, z_i)\}_{i=1,2,\cdots}, \text{ where } z_1,z_2,\cdots i.i.d.\leftarrow \pi\nonumber
\end{align}
To ensure low false positive rate $\alpha$ in each out world, the attack threshold $c_{\alpha}(\theta_0)$ must be smaller than $\alpha$ fraction of the non-member instances in $P^{\theta_0}$, i.e., for every possible target model $\theta_0$, 
\begin{equation}
    \label{eqn:approx_obj_pop_alpha}
    \frac{ \Big\lvert
    \left\{
        (\theta, z) \in P^{\theta_0}:\ell(\theta,x_z,y_z)\leq c_{\alpha}(\theta_0)
    \right\}
    \Big\rvert}{|P^{\theta_0}|} = \alpha
\end{equation}
The solution threshold for $c_{\alpha}(\theta_0)$ in \eqref{eqn:approx_obj_pop_alpha} (i.e., the ${\alpha\text{-percentile}}$ of the loss histogram for population data on the target model $\theta_0$) recovers the model-dependent attack threshold for Attack P. Observe that in the attack threshold $c_{\alpha}(\theta)$, there is no dependency on the target feature $x_z$ and label $y_z$. Therefore, attack P only reduces the uncertainty with regard to the target model $\theta$, but does not reduce the uncertainty that attacker has with regard to the target record. This is also reflected by the threshold values (or shapes of loss histograms) of Attack P for different targets  in Figure~\ref{fig:purchase100_2a_all_attack_loss_dist_plots_new}, which are the same in the two subplots targeting different data ($z_1,z_2$) and the same model, but are different between the two subplots targeting different models ($\theta_1,\theta_2$) and the same data.

\input{figure_scripts/purchase100_2a_all_attack_loss_dist_plots_ccs.tex}

\subsection{Attack R: sample-dependent MIA via \underline{R}eference models}

The privacy loss of the model with respect to the target data could be directly related to how susceptible the target data is to be memorized (e.g., being an outlier)~\cite{feldman2020does}.  Based on this finding, we further design the membership inference Attack R that applies different attack threshold $c_{\alpha}(x_z,y_z)$ for different target data sample (both its input features $x_z$ and the label $y_z$). This attack relies on training many reference models on population datasets (excluding the target data sample) and evaluating their loss on one specific record, as described in \citet{long2020pragmatic}.  This attack is also very similar to the membership inference attacks designed for summary statistics and graphical models, which use reference models to compute the probability of the null hypothesis~\cite{sankararaman2009genomic, murakonda2021quantifying}.  We define the test with sample-dependent attack threshold in our case as follows.
\begin{equation}
    \text{If } \ell(\theta,x_z,y_z)\leq c_{\alpha}(x_z,y_z), \text{ reject }H_0,\nonumber
\end{equation}
where $c_{\alpha}(x_z,y_z)$ is a threshold function that depends on the target data feature and label. Since the threshold is the same for all model $\theta$ (given fixed target data $z$), this strategy corresponds to the inference game for a fixed record (Definition~\ref{def:game_record}). Therefore, each out world (i.e., datasets, models and records generated under $b=0$ in the inference game) is as follows. 
\begin{align}
    P_{out}(D,\theta,z):\ & D \xleftarrow{n\ i.i.d. samples} \pi, \theta\leftarrow\mathcal{T}(D) \nonumber\\
    & z = z_0 = (x_{z_0},y_{z_0})\label{eqn:pout_R}
\end{align}
where $z_0$ could be any given target data. Compared to Attack S, the out world in \eqref{eqn:pout_R} is strictly smaller than \eqref{eqn:S_pout}, in the sense that the attacker's uncertainty about the feature $x_{z_0}$ of the target data is also reduced (besides uncertainty about the label $y_{z_0}$).

To obtain empirical samples from the out world, the attacker only need to train the following set of reference models (on random population datasets) and use the target data $z_0$.
\begin{align}
    R^{z_0}=\{(\theta_i,z_0)\}_{i=1,2,\cdots}, \text{ where } \theta_i\leftarrow\mathcal{T}(D_i) \text{ , } D_i\xleftarrow{n\ i.i.d. samples}\pi\nonumber
\end{align}
To ensure low FPR $\alpha$ of Attack R on the above empirical samples from the induced out world, we need to ensure that for every possible target data $z_0$, less than $\alpha$ fraction of the instances in $P$ incur smaller loss than the threshold $c_{\alpha}(x_{z_0},y_{z_0})$. That is
\begin{equation}
    \label{eqn:approx_obj_ref_alpha}
    \frac{ \lvert
    \left\{
        (\theta, z) \in R^{z_0}:\ell(\theta,x_z,y_z)\leq c_{\alpha}(x_{z_0},y_{z_0})
    \right\}
    \rvert}{|R^{z_0}|} = \alpha
\end{equation}
The solution threshold function $c_{\alpha}(x_z,y_z)$ of \eqref{eqn:approx_obj_ref_alpha} (i.e., the $\alpha$-percentile of the loss histogram for target data $z$ on reference models) gives the sample-dependent attack threshold $c_{\alpha}(x_z,y_z)$ for Attack R. Observe that in the attack loss threshold $c_{\alpha}(x_z, y_z)$, there is no dependency on the target model $\theta$. Therefore, it only reduces the uncertainty with regard to the target data feature $x_z$ and label $y_z$, but does not reduce which uncertainty the attacker has with regard to the target model $\theta$. And this can be seen from the threshold values (or shapes of loss histograms) of Attack R for different targets in Figure~\ref{fig:purchase100_2a_all_attack_loss_dist_plots_new}, which are the same in the two subplots targeting different models ($\theta_1,\theta_2$) and the same data, but are different between the two subplots targeting different data ($z_1,z_2$) and the same model.

\subsection{Attack D: model-dependent and sample dependent MIA via \underline{D}istillation}
Can we design a stronger attack that takes advantage of all the information available in the target model and the target data (despite only having black-box access to the target model and knowledge of population data points)? We design a membership inference Attack D whose threshold function $c_{\alpha}(D_{\theta},x_z,y_z)$ that depends on both the target record $z$ and the unknown target dataset $D_{\theta}$ for training the target model $\theta$, as follows.
\begin{equation}
    \text{If } \ell(\theta,x_z,y_z)\leq c_{\alpha}(D_{\theta},x_z,y_z), \text{ reject }H_0,
\end{equation}
% 
% However, the degree of freedom in the threshold function $c_{\alpha}(\theta,x_z,y_z)$ is still too large for us to directly solve \eqref{eqn:obj_general_alpha}. Therefore, we restrict $c_{\alpha}(\theta,x_z,y_z)$ to take the following form.
% %
% \begin{equation}
%     \label{eqn:threshold_aux_distill}
%     c_{\alpha}(\theta,x_z,y_z) = c_{\alpha}(D_\theta,x_z,y_z),
% \end{equation}
%
Here $D_{\theta}=\mathcal{T}^{-1}(\theta)$ is the unknown training dataset for target model $\theta$. For the simplicity of derivation, let us first assume that the randomized training algorithm $\mathcal{T}$ has a deterministic inverse mapping $\mathcal{T}^{-1}:\theta\rightarrow D$, i.e., the training dataset for a given model $\theta$ is uniquely specified. (Later we also show how to approximate the training datasets $D_{\theta}$ for model $\theta$ via relabelling population data records, when the training algorithm $\mathcal{T}$ is not invertible.) Then the inverse dataset $D_{\theta}$ takes the same value for all models $\theta$ trained from the same training dataset. Consequently, the threshold function $c_\alpha(D_{\theta},x_z,y_z)$ takes the same value on all the models retrained from the training dataset of a given target model. Therefore, this game corresponds to the inference game for a fixed pair of record and dataset (Definition~\ref{def:game_worstcase}), and each out world (i.e., targets generated under $b=0$ in the inference game) is as follows.
\begin{align}
    P_{out}(D,\theta, z):\ & D = D_{\theta_0}\setminus z_0 = \mathcal{T}^{-1}(\theta_0)\setminus z_0, \theta\leftarrow\mathcal{T}(D)\nonumber\\
    & z = z_0 = (x_{z_0}, y_{z_0}),\label{eqn:pout_D}
\end{align}

Here $\theta_0$ and $z_0$ could be any given pair of target model and target data. In this out world $P_{out}$, the only remaining uncertainty is the randomness of the training algorithm $\mathcal{T}$, while all the uncertainty with regard to the target dataset and target record is eliminated. 

To (approximately) sample from this out world, which contains models retrained on the same training dataset as the target model, our proposed attacker generates the following set of self-distilled models using the target model $\theta_0$ and target data $z_0$.
\begin{align}
    M^{\theta_0,z_0} & =\{(\theta_i, z_0)\}_{i=1,2,\cdots}, \text{ where } \theta_i\leftarrow\mathcal{T}(D'_i), \\
    &D'_i\xleftarrow{\text{soft-labeled with }\theta}D_i\xleftarrow{n\ i.i.d. samples}\pi
\end{align}
Here we are assuming that these distilled models $M=\{\theta_i\}_{i=1,2,\cdots}$ are good approximations for models retrained on the target model's training dataset (excluding the target data), which are in the out world \eqref{eqn:pout_D}. This is reasonable because the distillation dataset $D'_i$ consists of soft-labeled population data points by the target model $\theta$, and therefore roughly recovers the training dataset of the target model except for information about the target data.~\footnote{This is because under large population data pool, a target record is very unlikely to be sampled into the distillation dataset.}

To ensure low false positive rate $\alpha$, the threshold function needs to satisfy the following equation over distilled models.
\begin{equation}
    \label{eqn:approx_obj_distill_alpha}
    \frac{ \Big\lvert
    \left\{
        (\theta,z) \in M^{\theta_0,z_0}:\ell(\theta,x_z,y_z)\leq c_{\alpha}(\theta_0,x_{z_0},y_{z_0})
    \right\}
    \Big\rvert}{|M^{\theta_0,z_0}|} = \alpha
\end{equation}
By solving \eqref{eqn:approx_obj_distill_alpha} with the $\alpha$-percentile of the loss histogram for the target data on distilled models, we obtain the attack threshold function $c_{\alpha}(D_\theta, x_z, y_z)$ in Attack D. Observe that this threshold depends on both the target model $\theta$ and the target data $z$, therefore it reduces the attacker's uncertainties with regard to both the target data and the target model (dataset). This can be seen from the threshold values (or shapes of loss histograms) of Attack D for different targets in Figure~\ref{fig:purchase100_2a_all_attack_loss_dist_plots_new}, which are different between four pairs of target model and data. Because of this dependence on both target model and target data, Attack D more precisely measures the leakage of the model, which is desirable for privacy auditing.

\subsection{(Idealized) Attack L: \underline{L}eave-one-out attack}

\label{subsec:loo_explain}

An ideal attack, that removes the randomness over the training data (except the target data that could potentially be part of the training set) would be the \textit{leave-one-out} attack. In this attack, the adversary trains reference models $\theta'$ on $D \setminus \{(x_b, y_b)\}$, where the randomness only comes from the training algorithm $\mathcal{T}$. The attack would be in the same class of attacks as in Attack D, as it will be a model-dependent and data-dependent attack. It also\ runs a similar hypothesis test, however the attack requires assuming the adversary already knows exactly the $n-1$ data records in $D \setminus \{(x_b, y_b)\}$. This is a totally acceptable assumption in the setting of privacy auditing. However, for practical settings, this assumption may be too strong.

Note that Attack D aims at reproducing the performance of the leave-one-out attack without (impractically) assuming the knowledge of $n-1$ data records in $D \setminus \{(x_b, y_b)\}$.

\subsection{Summary and Attacks Comparison}
\label{ssec:summary_attack_explain}

For more accurately identifying whether a data point $z$ has been part of the training set of $\theta$, here are the main underlying questions for attacks we present in this section:

\begin{itemize}[leftmargin=1.25em]

	\item How likely is the loss $\ell(\theta, z)$ to be a sample from the distribution of loss for random population data on (Attack P: the same model) (Attack S: models trained on the population data)? We reject the null hypothesis depending on the tolerable false positive rate $\alpha$ and the estimated distribution of loss. 
	
	\item How likely is the loss $\ell(\theta, z)$ to be a sample from the distribution of loss for the target data $z$ on (Attack R: models trained on population data) (Attack D: models trained to be as close as possible to the target model, using distillation) (Attack L: models trained on $n-1$ records from $D$ excluding $z$)? We reject the null hypothesis d depending on the tolerable false positive rate $\alpha$ and the estimated distribution of loss.
	
\end{itemize}

Effectively, these questions cover different types of hypothesis tests that could be designed for performing \textit{confident} membership inference attacks under \textit{different uncertainties} about the out world (i.e., different inference games).  We expect these attacks to have different errors due to the uncertainties that can influence their performance.  Getting close to the performance of the ideal leave-one-out attack, without impractically assuming the adversary knows the data records in $D \setminus \{(x_b, y_b)\}$, is the ultimate goal of our attack. 

Attacks S and P are of the same nature. However, attack S could potentially have a higher error due to its imprecision in using other models to approximate the loss distribution of the target model on population data. Attacks R, D are also of the same nature. However, we expect attacks D to have more confidence in the tests due to reducing the uncertainty of other training data that can influence the model's loss.  Thus, we expect attack D to be the closest to the strongest (yet impractical) attack which is the leave-one-out attack.

\section{Empirical Evaluation}
\label{sec:empirical}

In this section, we empirically study the performance of different attacks.
Our goal is \textit{not} only to show that certain models trained on certain datasets are vulnerable to membership inference attacks, \textit{nor} only to show that our new attacks are stronger than previous attacks in a particular metric. Our goal is to understand:

\begin{enumerate}[leftmargin=2.25em]
    \item How does reduced uncertainty in different attacks improve their performance (under the same confidence) on \textit{general} and \textit{worst-case} target models and data records?
    \item How and why membership predictions (under the same confidence) \textit{differ} between two different attacks, when applied on the \textit{same} set of target models and data records?
    \item How vulnerable are the record that different attacks (with the same confidence) agree or disagree on? Here we refer to vulnerability as the performance of leave-one-out attack on models trained with and without this record (while fixing the remaining target dataset). Why do certain attacks miss extremely vulnerable data records?
\end{enumerate}

\begin{figure*}[h]
    \centering
    {\tikzset{external/export=false}
    \begin{subfigure}{.33\textwidth}
        \begin{tikzpicture}[scale=0.7]
            \begin{axis}
               [name=purchase1002aplot, 
                    xlabel={FPR}, ylabel={TPR},
                    xmin = 0, xmax = 1,
                    ymin = 0, ymax = 1, yscale=0.8,
                    xtick={0.0,0.2,0.4,0.6,0.8,1.0}, xticklabels={0.0,0.2,0.4,0.6,0.8,1.0},
                    grid = major, title style={yshift=0.9cm},
                    legend style={at={(1.0, 0.0)},anchor=south east}]
                % attack S
                        \addplot[solid, very thick, blue, no marks] table[skip first n=1,x index=1, y index=2, col sep=comma] {"data/tpr_vs_fpr_curves/purchase100_2a_ccs/attack_S_linear_all_models_tpr_vs_fpr_data.csv"};
                % attack P
                        \addplot[solid, very thick, green!50!black, no marks] table[skip first n=1,x index=1, y index=2, col sep=comma] {"data/tpr_vs_fpr_curves/purchase100_2a_ccs/attack_P_linear_all_models_tpr_vs_fpr_data.csv"};
                % attack R
                        \addplot[solid, very thick, orange, no marks] table[skip first n=1,x index=1, y index=2, col sep=comma] {"data/tpr_vs_fpr_curves/purchase100_2a_ccs/attack_R_linear_all_models_tpr_vs_fpr_data.csv"};
                % attack D
                    \addplot[solid, very thick, red, no marks] table[skip first n=1,x index=1, y index=2, col sep=comma] {"data/tpr_vs_fpr_curves/purchase100_2a_ccs/attack_D_linear_all_models_tpr_vs_fpr_data.csv"};
               \legend{Attack S (AUC = 0.754), Attack P (AUC = 0.756), Attack R (AUC = 0.808), Attack D (AUC = 0.831)}
            \end{axis}
        \end{tikzpicture}
        \caption{Overall TPR-FPR}
    \end{subfigure}\hfill
    \begin{subfigure}{.33\textwidth}
        \begin{tikzpicture}[scale=0.7]
            \begin{axis}
               [name=purchase1002aplot_low_FPR, 
                   ymode = log,
                   ymin = 0.0001, ymax = 1, yscale=0.8,
                   xmode = log,
                   xmin = 0.0001, xmax = 1,
                   xtick={0.0001,0.001,0.01,0.1,1.0}, xticklabels={$10^{-4}$,$10^{-3}$,$10^{-2}$,$10^{-1}$,$10^{0}$},
                   grid = major, title style={yshift= 0.9cm},
                   legend style={at={(1.0, -0.25)},anchor=south east}]
               % attack S
                       \addplot[solid, very thick, blue, no marks] table[skip first n=1,x index=1, y index=2, col sep=comma] {"data/tpr_vs_fpr_curves/purchase100_2a_ccs/attack_S_linear_all_models_tpr_vs_fpr_data.csv"};
               % attack P
                      \addplot[solid, very thick, green!50!black, no marks] table[skip first n=1,x index=1, y index=2, col sep=comma] {"data/tpr_vs_fpr_curves/purchase100_2a_ccs/attack_P_linear_all_models_tpr_vs_fpr_data.csv"};
               % attack R
                      \addplot[solid, very thick, orange, no marks] table[skip first n=1,x index=1, y index=2, col sep=comma] {"data/tpr_vs_fpr_curves/purchase100_2a_ccs/attack_R_linear_all_models_tpr_vs_fpr_data.csv"};
               % attack D
                    \addplot[solid, very thick, red, no marks] table[skip first n=1,x index=1, y index=2, col sep=comma] {"data/tpr_vs_fpr_curves/purchase100_2a_ccs/attack_D_linear_all_models_tpr_vs_fpr_data.csv"};
            \end{axis}
        \end{tikzpicture} 
        \caption[Network]{Focus on Small FPR Region}
    \end{subfigure}\hfill
    \begin{subfigure}{.33\textwidth}
        \begin{tikzpicture}[scale=0.7]
            \begin{axis}
               [name=purchase1002aplot_high_TPR, 
                    xmin = 0.2, xmax = 1,
                    ymin = 0.7, ymax = 1, yscale=0.8,
                    xtick={0.0,0.2,0.4,0.6,0.8,1.0}, xticklabels={0.0,0.2,0.4,0.6,0.8,1.0},
                    grid = major, title style={yshift=0.9cm},
                    legend style={at={(1.0, 0.58)},anchor=south east}]
               % attack S
                       \addplot[solid, very thick, blue, no marks] table[x=x, y=y, col sep=comma] {"data/tpr_vs_fpr_curves/purchase100_2a_ccs/attack_S_linear_all_models_tpr_vs_fpr_data.csv"};
               % attack P
                      \addplot[solid, very thick, green!50!black, no marks] table[skip first n=1,x index=1, y index=2, col sep=comma] {"data/tpr_vs_fpr_curves/purchase100_2a_ccs/attack_P_linear_all_models_tpr_vs_fpr_data.csv"};
               % attack R
                      \addplot[solid, very thick, orange, no marks] table[skip first n=1,x index=1, y index=2, col sep=comma] {"data/tpr_vs_fpr_curves/purchase100_2a_ccs/attack_R_linear_all_models_tpr_vs_fpr_data.csv"};
               % attack D
                    \addplot[solid, very thick, red, no marks] table[skip first n=1,x index=1, y index=2, col sep=comma] {"data/tpr_vs_fpr_curves/purchase100_2a_ccs/attack_D_linear_all_models_tpr_vs_fpr_data.csv"};
            \end{axis}
        \end{tikzpicture}    
        \caption{Focus at High TPR Region}
    \end{subfigure}
    }
    \caption{FPR vs TPR with AUC scores for all attacks on Purchase100 Dataset experimental setups II (details in Appendix B). Attack S, R and D use 999 shadow, reference and distilled models respectively, and Attack P uses 1000 population data points (per-class). The evaluated TPR and FPR are averaged over 10 target models for Attack S, P and R, while for Attack D we only average over 3 target models due to limited computational resources. Attack D achieves the highest AUC score, followed closely by Attack R. For small FPR in subplot (b), Attack R has slightly higher TPR than Attack D, which is more than 10x higher than that of Attack S and P. (We compute attack threshold for small $\alpha<0.001$ by smoothening the loss histogram with linear interpolation method, as explained in Appendix B.3.) Meanwhile, for high TPR region in subplot (c), at any fixed high TPR ($>0.7$), Attack D enables roughly 2x smaller FPR than Attack S, P and R. }
    \label{fig:purchase100_fpr_vs_tpr_plots_new}
\end{figure*}
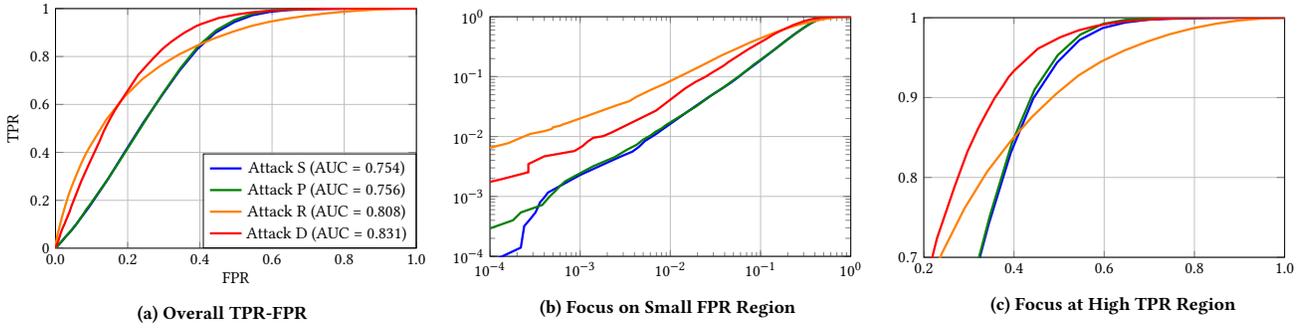

Here we report the attack results on models trained using the Purchase100 dataset. We include evaluations on more datasets and experiment setups in Appendix C. We implement our experiments in the open source privacy meter tool and release our code.~\footnote{\url{https://github.com/privacytrustlab/ml_privacy_meter/tree/master/research/2022_enhanced_mia}}

\subsection{Illustration of Attack Threshold}

How differently does the attack threshold in Attack S, P, R, D depend on the target model and the target data? How do the dependencies of the attack threshold on model and data sample affect an attack's success? To investigate these questions, we plot the loss histograms that different attacks use to compute the thresholds for two randomly chosen target models on two different records in Figure~\ref{fig:purchase100_2a_all_attack_loss_dist_plots_new}. As complement, we also show distributions of the
threshold chosen under more pairs of different target models and target records (under the same level of
FPR) in Appendix C.1. 

We observe that the attack threshold's dependency on model and record make its loss histogram more concentrated (i.e., reduce uncertainty), thus forming sharper attack signals and enabling attack success. In Figure~\ref{fig:purchase100_2a_all_attack_loss_dist_plots_new}, \textbf{Attack S} uses constant threshold for attacking all four different targets, because it only estimates how likely does a \textit{random} nonmember data record of a \textit{random} target model (trained from population data) incur small loss. This is overly general and makes Attack S wrong for all four targets. \textbf{Attack P} uses different thresholds for different target models, but still uses the same threshold for different target data ($z_1$ and $z_2$). Consequently, Attack P still wrongly predicts all four targets. On the contrary, \textbf{Attack R} considers how likely does a \textit{particular} target data record $z$ incur small loss  on reference models, and constructs different (more concentrated) out worlds that depend on target data. Moreover, \textbf{Attack D} obtains even more concentrated out worlds, by considering how likely does a \textit{particular} target record $z$ incur small loss on distilled models (which approximate leave-one-out models). Due to these sharper attack signals (reduced uncertainty), Attack D predicts all four targets correctly with high confidence  (i.e., its thresholds are well above the target's loss).

\subsection{Evaluation of Attack Performance}
\label{ssec:evaluation_attack_performance}

Which attack (S, P, R, or D) provides the best performance? How do we evaluate the strength of an attack besides using its accuracy? How can we design fine-grained attack evaluation metrics for comparing the performance of different attacks under the same level of confidence (FPR)?

\begin{figure*}[h!]
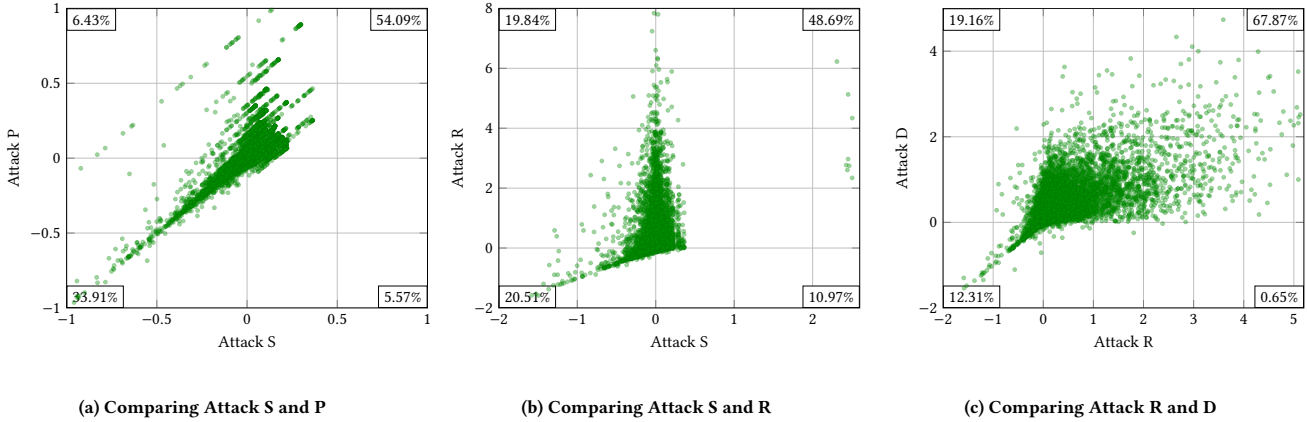

    \centering
    {
     \begin{subfigure}[b]{.3\textwidth}
     \if\compileScatterFigures1
     \begin{tikzpicture}[
         scale=0.7,
         every mark/.append style={mark size=1pt, fill=green!60!black, draw=green!50!black, opacity=0.4}
     ]
         \begin{axis}
             [name=attackSPtrainpointsplot,
             xlabel={Attack S}, ylabel={Attack P},
             ymin = -1.0, ymax = 1.0,
             xmin = -1.0, xmax = 1.0,
             grid = major]

             \addplot[only marks] table[skip first n=1,x index=1, y index=2, col sep=comma] {"data/attack_comparison_scatterplots/purchase100_2a/models_0_1_2__0_3_attack_0_2_train_data.csv"};

             \node at (axis cs:1,1) [anchor=north east, draw=black, fill=white] {54.09\%};
             \node at (axis cs:-1,1) [anchor=north west, draw=black, fill=white] {6.43\%};
             \node at (axis cs:-1,-1) [anchor=south west, draw=black, fill=white] {33.91\%};
             \node at (axis cs:1,-1) [anchor=south east, draw=black, fill=white] {5.57\%};
         \end{axis}
     \end{tikzpicture}
     \else
        \includegraphics[]{figScatter/\filename-figure\thescatterFigureNumber.pdf}
        \stepcounter{scatterFigureNumber}
     \fi
     \caption{Comparing Attack S and P}
     \end{subfigure}\hfill
     \begin{subfigure}[b]{.3\textwidth}
     \if\compileScatterFigures1
     \begin{tikzpicture}[
         scale=0.7,
         every mark/.append style={mark size=1pt, fill=green!60!black, draw=green!50!black, opacity=0.4}
     ]
         \begin{axis}
             [name=attackSRtrainpointsplot,
             xlabel={Attack S}, ylabel={Attack R},
             ymin = -2.0, ymax = 8.0,
             xmin = -2.0, xmax = 2.6,
             grid = major]

             \addplot[only marks] table[skip first n=1,x index=1, y index=2, col sep=comma] {"data/attack_comparison_scatterplots/purchase100_2a/models_0_1_2__0_3_attack_0_1_train_data.csv"};

             \node at (axis cs:2.6,8) [anchor=north east, draw=black, fill=white] {48.69\%};
             \node at (axis cs:-2,8) [anchor=north west, draw=black, fill=white] {19.84\%};
             \node at (axis cs:-2,-2) [anchor=south west, draw=black, fill=white] {20.51\%};
             \node at (axis cs:2.6,-2) [anchor=south east, draw=black, fill=white] {10.97\%};
         \end{axis}
     \end{tikzpicture}
     \else
        \includegraphics[]{figScatter/\filename-figure\thescatterFigureNumber.pdf}
        \stepcounter{scatterFigureNumber}
     \fi
     \caption{Comparing Attack S and R}
     \end{subfigure}\hfill
     \begin{subfigure}[b]{.3\textwidth}
     \if\compileScatterFigures1
     \begin{tikzpicture}[
         scale=0.7,
         every mark/.append style={mark size=1pt, fill=green!60!black, draw=green!50!black, opacity=0.4}
     ]
         \begin{axis}
             [name=attackRDtrainpointsplot,
             xlabel={Attack R}, ylabel={Attack D},
             ymin = -2.0, ymax = 5.0,
             xmin = -2.0, xmax = 5.2,
             grid = major]

             \addplot[only marks] table[skip first n=1,x index=1, y index=2, col sep=comma] {"data/attack_comparison_scatterplots/purchase100_2a/models_0_1_2__0_3_attack_1_3_train_data.csv"};

             \node at (axis cs:5.2,5) [anchor=north east, draw=black, fill=white] {67.87\%};
             \node at (axis cs:-2,5) [anchor=north west, draw=black, fill=white] {19.16\%};
             \node at (axis cs:-2,-2) [anchor=south west, draw=black, fill=white] {12.31\%};
             \node at (axis cs:5.2,-2) [anchor=south east, draw=black, fill=white] {0.65\%};
         \end{axis}
     \end{tikzpicture}
     \else
        \includegraphics[]{figScatter/\filename-figure\thescatterFigureNumber.pdf}
        \stepcounter{scatterFigureNumber}
     \fi
     \caption{Comparing Attack R and D}
     \end{subfigure}\hspace*{\fill}
    }
    \caption{Scatter plot comparing the membership prediction of different attacks (with the same FPR requirement $\alpha = 0.3$) on the training dataset of a target model in Purchase100 II setup. Each dot on the scatter plot corresponds to a particular training data record. Each coordinate of the dot equals the loss threshold used by a particular attack on the target data minus the loss of the target data (on the target model). Therefore, the coordinate approximates the attack's confidence that a given target data is member. We compare two attacks in each subplot through x-axis and y-axis: Plot (a) compares Attacks S and P; plot (b) compares Attack S and R; plot (c) compares Attacks R and D. The percentage number in each corner stands for the fraction of training data records whose confidence values land in corresponding regions (partitioned by the $x$-axis and $y$-axis). Attacks that are similar to each other would produce membership predictions with correlated confidence, thus incurring more points in the diagonal regions of the scatter plot, i.e. northeast and southwest. If the y-axis attack is stronger, the scatter plot would shift more towards the northwest region (where the y-axis attack is correct while x-axis attack is wrong) than the southeast region.}
    \label{fig:purchase100_2a_attack_comparison_scatterplots_train}
\end{figure*}

\paragraph{Improved Average Performance of Attacks.} We quantify the attacker's \textit{average} performance on \textit{general} targets using two metrics: its true positive rate (TPR), and its false positive rate (FPR), over the \textit{random} member and non-member data of \textit{random} target models. We use the ROC curve to capture the tradeoff between the TPR and FPR of an attack, as its threshold $c_{\alpha}$ is varied across different FPR tolerance $\alpha$. The AUC (area under the ROC curve) score then measures the strength of an attack. We plot the ROC curves of all attacks on the Purchase100 dataset, and compute their AUC (area under the ROC curve) score in Figure~\ref{fig:purchase100_fpr_vs_tpr_plots_new}. The attack with the highest AUC score on Purchase100 is Attack D, which has the least level of uncertainty, as discussed Section~\ref{section3}. We further show in Appendix C that this trend of AUC scores holds for different Purchase100 training setups, and different datasets such as CIFAR10, CIFAR100 and MNIST. This shows that reducing uncertainties (in Attack R and D) effectively improves attack performance, when compared to Attack S and P designed for overly general targets.

\paragraph{Improved Performance of Attacks in High Confidence Regime} Besides the AUC score, we also observe that in small FPR region (FPR$<0.2$) of Figure~\ref{fig:purchase100_fpr_vs_tpr_plots_new} (b), Attack R and Attack D has significantly (above 10x) higher TPR than Attack S and P under a fixed low FPR. However, Attack R performs slightly better than Attack D in terms of TPR at small FPR. We believe this is because the additional approximation error in 
Attack D, for using distilled models to \textit{approximate} retrained
models. We believe it is an important followup work to improve this approximation
quality of Attack D at small FPR, while still being able to extract information
about the unknown remaining target dataset (to get close to the performance of idealized impractical Attack L). Moreover, it is interesting that for fixed \textit{high TPR}, Attack D has a much better performance (lower FPR) than any other attacks (S, P and R) in Figure~\ref{fig:purchase100_fpr_vs_tpr_plots_new} (c). Note that differentially private algorithms aim to (roughly speaking) bound the total error (FPR and FNR: 1-TPR). So, both high TPR and low FPR are important for privacy risk. These high performances of Attack R and D at high confidence region (low FPR or high TPR) shows the benefits of reducing uncertainty on attack performance.
    
\begin{table*}[!h]
    \caption{Agreement rate between ground truth (GT) membership values, and Attacks L, S, P, R, D for 500 train and 500 test data points. The upper triangle of the table corresponds to the agreement rates of train data points, whereas the lower triangle corresponds to the agreement rates of test data points. The experimental setup is Purchase100 II, with effective $FPR \approx 0.05$ (left table) and effective $FPR \approx 0.3$ (right table).}\label{tab:purchase100_2a_train_test_agreement_with_leave_one_out}
    \begin{subtable}{.5\linewidth}\centering
        \caption{Agreement between Attacks with FPR 0.05}\label{tab:agreement_subtable_fpr005}
    {
    \begin{tabular}{|c|c|c|c|c|c|c|}
        \hline
                    & \textbf{S} & \textbf{P} & \textbf{R} & \textbf{D} & \textbf{L} & \textbf{GT} \\
        \hline
        \textbf{S} &             & 0.972      & 0.774      & 0.822       & 0.246     & 0.066        \\
        \hline
        \textbf{P}  & 0.968       &            & 0.770     & 0.834      & 0.238      & 0.066        \\
        \hline
        \textbf{R}  & 0.948        & 0.944      &           & 0.792      & 0.360     & 0.204        \\
        \hline
        \textbf{D}  & 0.964       & 0.956      & 0.948      &            & 0.352    & 0.180        \\
        \hline
        \textbf{L}  & 0.914       & 0.918      & 0.934      & 0.930      &          & 0.816        \\
        \hline
        \textbf{GT}  & 0.946        & 0.950       & 0.950      & 0.950      & 0.944     &              \\
        \hline
    \end{tabular}
    % 'attack_0_alpha': 0.038,
    % 'attack_1_alpha': 0.051,
    % 'attack_2_alpha': 0.034,
    % 'attack_3_alpha': 0.001,
    % 'attack_4_alpha': 0.05,
    }
    \end{subtable}%
    \begin{subtable}{.5\linewidth}\centering
        \caption{Agreement between Attacks with FPR 0.3}\label{tab:agreement_subtable_fpr03}
    {
    \begin{tabular}{|c|c|c|c|c|c|c|}
        \hline
                    & \textbf{S} & \textbf{P} & \textbf{R} & \textbf{D} & \textbf{L} & \textbf{GT} \\
        \hline
        \textbf{S} &             & 0.892      & 0.534      & 0.644      & 0.592       & 0.576       \\
        \hline
        \textbf{P}  & 0.932       &            & 0.526      & 0.668      & 0.600      & 0.592       \\
        \hline
        \textbf{R}  & 0.750       & 0.742      &            & 0.814      & 0.754      & 0.730       \\
        \hline
        \textbf{D}  & 0.878       & 0.894      & 0.804       &            & 0.804      & 0.78      \\
        \hline
        \textbf{L}  & 0.692        & 0.672      & 0.738      & 0.714      &            & 0.968        \\
        \hline
        \textbf{GT}  & 0.700       & 0.696      & 0.698      & 0.694      & 0.700      &            \\
        \hline
    \end{tabular}
    % 'attack_0_alpha': 0.255,
    % 'attack_1_alpha': 0.265,
    % 'attack_2_alpha': 0.265,
    % 'attack_3_alpha': 0.026,
    % 'attack_4_alpha': 0.3,
    }
    \end{subtable}
\end{table*}

\subsection{Detailed Comparison of Different Attacks}
\label{ssec:attack_comparison}

Besides attack strength, how differently are the attacks performing on the \textit{same} set of input target models and target points? How do the reduced uncertainties of attacks affect the attacker's membership predictions quantitatively (in accuracy) and qualitatively? How often do the membership predictions of two different attacks agree with each other? Do the attacks (with different amounts of reduced uncertainty) have different confidence on the same input (being a member)? How far away are the attack performances from the most ideal leave-one-out attacks that reduces all uncertainty (described in Section~\ref{section3})? Answers to these questions require understanding how and why the attacks perform differently, for which we do detailed comparisons between attacks as follows.

\paragraph{Similarity of Attacks with Each Other in Predictions and Confidence.} In Figure~\ref{fig:purchase100_2a_attack_comparison_scatterplots_train}, we show the scatter plot that compares the membership predictions and confidence of different attacks, on the training data of a target model in Purchase 100 II experiment setup (described in Appendix B). The scatter plot Figure~\ref{fig:purchase100_2a_attack_comparison_scatterplots_train} (a) compares Attack S and P, and shows that they make similar membership predictions with similar confidence (as the plot is correlated around the diagonal line). However, Attack P performs slightly better than Attack S, because there are slightly more points in northwest region than the southeast region of the scatter plot Figure~\ref{fig:purchase100_2a_attack_comparison_scatterplots_train} (a), which Attack P predicts correctly while Attack S fails. Meanwhile, in Figure~\ref{fig:purchase100_2a_attack_comparison_scatterplots_train} (b), we observe that Attack S and R makes very different membership predictions, because the plot is far away from the diagonal line. Attack R is also significantly stronger than Attack S, because the northwest region contains a large fraction (19.84\%) of inputs (which Attack R correctly predicts as member while Attack S fails). Moreover, in northeast and southwest region both Attack S and R are making correct membership predictions, but Attack R tend to have higher confidence since most of the points lie above the diagonal line. Lastly, we observe from Figure~\ref{fig:purchase100_2a_attack_comparison_scatterplots_train} (c) that Attack D dominates Attack R for correctly guessing membership of training points, because the plot is shifted towards the northwest region.

\paragraph{Gap Between Attacks and the Ground Truth.} From Table~\ref{tab:purchase100_2a_train_test_agreement_with_leave_one_out} (b), among all attacks, Attack D agrees with the ground truth the most on train points (upper triangle) under confidence requirement FPR=$0.3$. This matches our observation in Figure~\ref{fig:purchase100_2a_all_attack_loss_dist_plots_new}, that Attack D has the largest threshold among four attacks under the same confidence requirement $\alpha$, thus correctly predicting more points as members without harming its confidence about FPR. 
We also observe that the agreement rate between Attack S and Attack P is as high as $0.9$ in Table~\ref{tab:purchase100_2a_train_test_agreement_with_leave_one_out}. This matches their linear comparison scatter plot in Figure~\ref{fig:purchase100_2a_attack_comparison_scatterplots_train}, and shows that Attack S and P are very similar in nature. This is consistent with our discussion in Section~\ref{section3} (i.e., Attack S and Attack P reflect similar uncertainty of the adversary).

\paragraph{Closeness of Attacks to Ideal Leave-one-out Attack.} Another interesting observation from Table~\ref{tab:purchase100_2a_train_test_agreement_with_leave_one_out} (b) is that, among all the attacks, Attack D agrees with the Attack L the most on predicting membership of training points (upper triangle), with agreement rate $0.804$. This is consistent with our understanding that Attack D (distillation) is highly similar in nature with Attack L (leave-one-out), by approximating the training dataset of a target model and performing retraining (via distillation), as discussed in Section~\ref{section3}.

\begin{figure*}
    \begin{minipage}[c]{0.325\linewidth}
        \resizebox{\linewidth}{!}{%
        \if\compileFigures1
        \begin{tikzpicture}
            \begin{axis}
               [name=purchase1002aplot,
                   xlabel={FPR}, ylabel={TPR},
                   ymin = -0.05, ymax = 1.1, yscale=0.8,
                   xtick={0.0,0.2,0.4,0.6,0.8,1.0}, xticklabels={0.0,0.2,0.4,0.6,0.8,1.0},
                   grid = major, title style={yshift=0.9cm},
                   legend style={at={(1.0, 0.0)},anchor=south east}]
                       \addplot[solid, very thick, blue] table[skip first n=1,x index=1, y index=2, col sep=comma] {"data/hard_examples/aggregated_hard_examples/diff_points/avg_attack_loo_results_50_loo_models_alpha_0.3_num_hard_example_10.csv"};
                    \addplot[solid, very thick, red] table[skip first n=1,x index=1, y index=2, col sep=comma] {"data/hard_examples/aggregated_hard_examples/diff_points/avg_attack_loo_results_50_loo_models_alpha_0.3_num_all_correct_example_10.csv"};
                    \addplot[solid, very thick, green!50!black] table[skip first n=1,x index=1, y index=2, col sep=comma] {"data/hard_examples/aggregated_hard_examples/diff_points/avg_attack_loo_results_50_loo_models_alpha_0.3_num_average_example_50_10.csv"};
                      \addplot[solid, very thick, orange] table[skip first n=1,x index=1, y index=2, col sep=comma] {"data/hard_examples/aggregated_hard_examples/diff_points/avg_attack_loo_results_50_loo_models_alpha_0.3_num_sp_correct_example_10.csv"};
        
                    \addlegendentry{R Correct (AUC = 0.984)}
                    \addlegendentry{All Correct (AUC = 0.899)}
                    \addlegendentry{Random (AUC = 0.805)}
                    \addlegendentry{SP Correct (AUC = 0.741)}
            \end{axis}
        \end{tikzpicture}
        \else
        \includegraphics[width=0.3\linewidth]{fig/\filename-figure\thefigureNumber.pdf}
        \stepcounter{figureNumber}
        \fi
        }
        \caption{Vulnerabilities of different types of records, in terms of the performance of Attack L on $50$ models trained with and without each record in each type (in Purchase100 II setup). The records are found via attacking $50$ target models with confidence $\alpha = 0.3$ (described in Section~\ref{ssec:why_vulnerable}). We observe that R correct records are the most vulnerable (of four types), indicating that Attack R and D are stronger than S and P in identifying vulnerable data.}
        \label{fig:purchase100_hard_example_fpr_vs_tpr_plot}
    \end{minipage}
    \hfill
    \begin{minipage}[c]{0.32\linewidth}
        \resizebox{\linewidth}{!}{%
        \if\compileFigures1
            \begin{tikzpicture}
            \begin{axis}[
                ybar, yscale=0.8,
                xlabel={log(loss)}, ylabel={density},
                legend style={at={(0.0, 1.1)},anchor=north west},
            ]
            % \addlegendimage{empty legend}

            % histograms
            \addplot+[hist={data=x, density, bins=50}, color=black, fill=blue, opacity=0.4] table[skip first n=1,x index=2, col sep=comma] {"data/diff_points_loss_histograms/purchase100_2a/loo_0_hard_examples_loss_values_alpha_0_3.csv"};
            \addplot+[hist={data=x, density, bins=75}, color=black, fill=red, opacity=0.4] table[skip first n=1,x index=2, col sep=comma] {"data/diff_points_loss_histograms/purchase100_2a/loo_0_all_correct_examples_loss_values_alpha_0_3.csv"};
            \addplot+[hist={data=x, density, bins=75}, color=black, fill=green!50!black, opacity=0.4] table[skip first n=1,x index=2, col sep=comma] {"data/diff_points_loss_histograms/purchase100_2a/loo_0_all_loss_values_alpha_0_3.csv"};
            \addplot+[hist={data=x, density, bins=60}, color=black, fill=orange, opacity=0.4] table[skip first n=1,x index=2, col sep=comma] {"data/diff_points_loss_histograms/purchase100_2a/loo_0_sp_correct_examples_loss_values_alpha_0_3.csv"};

            \addlegendentry{R Correct}
            \addlegendentry{All Correct}
            \addlegendentry{Random}
            \addlegendentry{SP Correct}
            \end{axis}
            \end{tikzpicture}
        \else
            \includegraphics[]{fig/\filename-figure\thefigureNumber.pdf}
            \stepcounter{figureNumber}
        \fi
        }
        \caption{Loss histogram of different types of records. The records are found via attacking $50$ target models (trained on Purchase 100 II setup) under confidence $\alpha = 0.3$, as described in Section~\ref{ssec:why_vulnerable}. We observe that R correct records generally have higher loss than SP correct records (and All correct records). This suggests that Attack R and D could identify vulnerable training records with intrinsically higher loss while S and P miss them. }
        \label{fig:purchase100_hard_example_histogram}
    \end{minipage}%
    \hfill
    \begin{minipage}[c]{0.325\linewidth}
        \includegraphics[width=\linewidth]{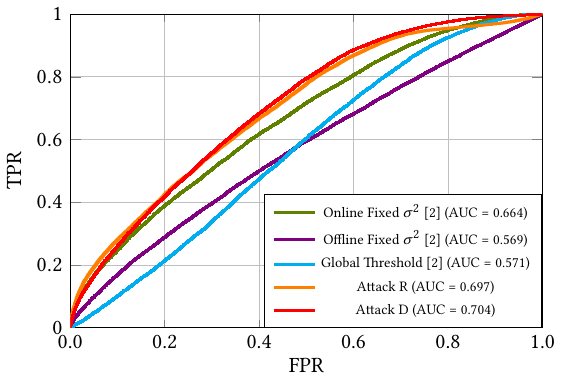}
    \caption{Comparison between our Attack R, D with the concurrent attacks~\cite{carlini2022membership} for CIFAR10 Setup IV. The TPR and FPR are evaluated over all member and non-member records of 1 target model. Due to computational constraint, we only use 29 models for constructing each attack. Compared to~\cite{carlini2022membership}, Attack R and D achieve higher AUC scores, and Attack R also achieves similar TPR at small FPR (Figure~10 in Appendix C.4).}
    \label{fig:related_work_carlini}
    \end{minipage}%
\end{figure*}

\subsection{Why Some Vulnerable Points are Not Detected by Certain Attacks}

\label{ssec:why_vulnerable}

In previous sections, we observed that Attack R and D identify significantly more training data as members than Attack S and P, while having the same confidence (in the first rows of Table~\ref{tab:purchase100_2a_train_test_agreement_with_leave_one_out}). In this section, we more closely investigate \textit{how} different attacks identify records with \textit{different levels of vulnerability}. Here we refer to the \textit{vulnerability} of a data
record as how indistinguishable the models trained with and without a given record
are (under a fixed remaining dataset) (i.e., the leave-one-out setting
studied in Attack L and memorization literature~\cite{feldman2020does,van2021memorization}). Therefore, we use the performance of the leave-one-out attack (Attack L) as a crucial tool for accurately estimating this indistinguishability or vulnerability of identified record. We look at multiple target models trained on the \textit{same} training dataset, and perform attacks on all the training data. By aggregating predictions by \textit{different} attacks~\footnote{\label{D_computation_cost}In principle, we should have also examined whether Attack D misses the record, but due to computation constraint, here we only performed Attack S, P and R.}, we divide the whole training dataset into the following four types of records.

\begin{itemize}[leftmargin=1.25em]
	\item \textbf{All correct:} records that are correctly identified by all the attacks (S, P, R and D) as member on most ($80\%$) target models. (For our Purchase 100 II setup, we identified $\approx$500 such records.)
	\item \textbf{R correct:} records that correctly identified by stronger attack (R) as member on most ($80\%$) target models, while being missed by weaker attacks (S and P) on most ($80\%$) target models. (For our Purchase 100 II setup, we identified $\approx$450 such records.)
    \item \textbf{SP correct:} records that correctly identified by weaker attacks (S and P) as member on most ($80\%$) target models, while being missed by stronger attack (R) on most ($80\%$) target models. (For our Purchase 100 II setup, we identified $\approx$125 such records.)
	\item \textbf{(Baseline) Random:} records randomly sampled from the training dataset (consisting of 5000 records) of the target models. 
\end{itemize}

We then examine \textit{how} and \textit{why} these different types of records incur \textit{different} level of vulnerabilities.

\input{figure_scripts/purchase100_2a_worst_case_neighbouring_loss_histogram}

\paragraph{Improved Ability of Attacks to Identify Vulnerable Records} Are the records identified by Attack R and D more vulnerable than the records identified by Attack S and P? We expect this because Attack R and D better capture the uncertainties associated with individual target data records, and therefore should more precisely estimate the privacy risk of individual target data records. To measure how vulnerable each identified record actually is, we perform the strongest impractical Attack L on leave-one-out models trained with and without this record (while fixing the remaining dataset to be the same as the target models), as explained in the beginning of Section~\ref{ssec:why_vulnerable}. 
The resulting attack performance, when aggregated over a set of records identified by a particular attack, then serves to quantify the ability of this attack to identify vulnerable records.

In Figure~\ref{fig:purchase100_hard_example_fpr_vs_tpr_plot}, we observe that, the \textbf{R correct} records (that are identified by Attack R while being missed by Attack S and P under FPR 0.3) are more vulnerable than \textbf{SP correct} records (that are identified by Attack S and P while being missed by Attack R). Moreover, both \textbf{R correct} records and \textbf{All correct} records are significantly more vulnerable than \textbf{SP correct} records and \textbf{random training data} of the target model. This shows that Attack R has a stronger ability to identify vulnerable data than Attack S and P, in the sense that it not only identifies more records (Table~\ref{tab:agreement_subtable_fpr03} first rows), but also more credibly estimate the vulnerability of individual target data (because R correct records are the most vulnerable in Figure~\ref{fig:purchase100_hard_example_fpr_vs_tpr_plot}).

\paragraph{Why Are Certain Vulnerable Records Missed by Attack S and P} We have shown that the \textbf{R correct} data records, while being extremely vulnerable and successfully detected by Attack R, are \textit{not} identified by Attack S and P as vulnerable. To understand why, we plot the loss histogram for \textbf{R correct} data records in Figure~\ref{fig:purchase100_hard_example_histogram}, and compare it with other types of records.
We observe that \textbf{R correct} records have higher loss values than \textbf{SP correct} and \textbf{All correct} records. Intuitively, this suggests that Attack S and P may fail to identify vulnerable training data that \textit{intrinsically} incur \textit{high} loss. This discovered property of \textbf{R correct} training records, is consistent with the intuition of "hard" examples that are accounted for in recent works~\cite{watson2021importance,carlini2022membership} to design attacks with improved performance. However, we offer a more rigorous way of identifying these worst-case records (as captured by stronger Attack R while being missed by Attack S and P).

\paragraph{Why certain records identified by Attack S and P are invulnerable} In Figure~\ref{fig:purchase100_hard_example_fpr_vs_tpr_plot}, we observe that the "SP
correct records" have the lowest vulnerability under leave-one-out attack. This shows that when a vulnerable record is identified by S and P, it is also nearly always identified by Attack R. Consequently, the remaining SP correct records  (identified by S and P as members while not identified by R
and D) are even less vulnerable than the baseline random record. This low vulnerability of SP correct records reflects mistakes that Attack S and P make in estimating record vulnerability, and shows that Attack S and P are relatively weak (compared to Attack R) in identifying vulnerable records. Consequently, whether a training record is identified by Attack S and P may serve as wrong indication for the vulnerability of this record. 

\paragraph{The Effect of Remaining Training Dataset on Data Vulnerability}

Does the same data record have a different level of vulnerability on target models trained on \textit{different} training datasets? We perform Attack R (with fixed confidence $\alpha$) on a given worst-case vulnerable (R correct) record $z$, and multiple target models trained on randomly sampled population data combined with $z$. We divide all the target models into two sets: the models where $z$ is \textit{correctly predicted} as member by Attack R, i.e. is vulnerable; and the models where $z$ is \textit{incorrectly predicted} as non-member, i.e., is not vulnerable.
By design of Attack R, the loss of record $z$ is smaller on models where it is correctly predicted, as we observe in Figure~\ref{fig:purchase100_worst_case_neighbouring_loss_histogram} (a).

Meanwhile, in Figure~\ref{fig:purchase100_worst_case_neighbouring_loss_histogram} (d), we observe that the loss of the target models on \textit{random training records}, remains roughly the \textit{same} no matter whether the target models are correctly or incorrectly predicted (for membership of $z$). This suggests that \textit{a large fraction} of the training dataset of a target model, may be independent of whether the given record $z$ is vulnerable. However, in Figure~\ref{fig:purchase100_worst_case_neighbouring_loss_histogram} (b) (c), we observe that the loss of the \textit{latent neighbors} of the vulnerable record (in the training dataset of a model), \textit{differs} dramatically between models that are correctly predicted or incorrectly predicted. This says that a target record may be more vulnerable on datasets that contain similar records to itself. This also suggests that latent neighbors of a record may have more visible influence (compared to the whole training dataset) on the vulnerability of this record. We leave it as an interesting open question, whether reducing uncertainties about the target record's latent neighbors yields a strong attack that is more efficient than Attack D (which ``wastefully'' approximates the whole dataset).

\paragraph{Comparison with Concurrent Work}

\label{ssec:comparison_with_carlini}

Even though our main goal in this paper is not to show that we have better heuristics to design stronger attacks, we show that our systematic methodology that resulted in attack R and D indeed is more powerful than the latest results~\cite{carlini2022membership}. Thanks to the code released by authors, we compared with their online and offline Likelihood Ratio Attack (LiRA) on the CIFAR10 dataset and Wide Resnet (with depth 28 and width 2) target models trained on 25000 data points in Figure~\ref{fig:related_work_carlini}. 

The LiRA attack~\cite{carlini2022membership} is similar to our Attack R, which exploits the dependency of
vulnerability on different records. However, due to our more precise
characterization of threshold dependency, Attack R achieves higher
performance than LiRA~\cite{carlini2022membership} in terms of AUC score (Figure~\ref{fig:related_work_carlini}), while achieving similarly high TPR at small fixed FPR. We do this also with a lower computation cost in terms of training reference models. \cite{carlini2022membership} needs $n$ OUT models for attacking all points, and $n$ IN models for attacking each point. We
only need $2\cdot n$ OUT models. So, for attacking $k$ points, their amortized cost is $(k+1)/2$ times more than our cost. Finally, our attack D additionally captures the influence of (unknown) target dataset on the privacy risk of a target instance, which LiRA~\cite{carlini2022membership} does not consider. Consequently, our Attack D achieves a better AUC
score than LiRA~\cite{carlini2022membership} (Figure~\ref{fig:related_work_carlini}), as well as a lower FPR at fixed given high TPR.

\section{Conclusions}
We provide a framework for auditing the privacy risk of a machine learning model about individual data records, through membership inference attacks that can guarantee FPR confidence over fine-grained (non-member) out worlds. Within this framework, we derive increasingly strong attacks, which highlight various uncertainties of the attacker that limit attack performance. %To show the effectiveness of reducing attack uncertainties, we empirically validate the improved performance of our attacks against models trained on benchmark datasets. Finally, we demonstrate novel usage of our attacks (with increasing strength), for analyzing the "differential vulnerability" of data records, and why some vulnerable records are missed by certain attacks. These insights enable more precise privacy risk estimates of (and comparison between) different data records in machine learning models.
\begin{acks}
The authors would like to thank Hongyan Chang and anonymous reviewers for helpful discussions on drafts of this paper. This research is supported by Google PDPO faculty research award,  Intel within the www.private-ai.org center, Meta faculty research award,  the NUS Early Career Research Award (NUS ECRA award number NUS ECRA FY19 P16), and the National Research Foundation, Singapore under its Strategic Capability Research Centres Funding Initiative. Any opinions, findings and conclusions or recommendations expressed in this material are those of the author(s) and do not reflect the views of National Research Foundation, Singapore.
\end{acks}

% \section*{Appendix}

% See all the appendix in \cite{ye2022enhanced}.

\bibliographystyle{ACM-Reference-Format}
\bibliography{reference}

\appendix
\section{Detailed derivation of approximated LRT for membership inference}
\label{appendix:lrt}

We first prove two useful approximation inequalities about the posterior distribution $P(\theta|D)$ as follows. 
\begin{enumerate}[leftmargin=2em]
    \item For arbitrary data point $z=(z_x,z_y)$, and arbitrary dataset $D$, we have
    \begin{equation}
        \label{eqn:approx_add_point}
        P(\theta|D\cup z) \geq e^{-\frac{\ell(\theta,z_x,z_y)}{T}}\cdot P(\theta|D).
    \end{equation} 
    \item Let $z_1,z_2,\cdots,z_n$ be i.i.d. samples from the data distribution $\pi(z)$. Then when $n$ is large enough, for any model parameter $\theta$, we have
    \begin{equation}
        \label{eqn:approx_expect_model_dist}
        \mathbf{E}_{z_1,\cdots,z_n\sim\pi}[P(\theta|z_1,\cdots,z_n)]\approx\mathbf{E}_{z_1,\cdots,z_{n-1}\sim\pi}[P(\theta|z_1,\cdots,z_{n-1})]
    \end{equation}
\end{enumerate}

\textbf{Proof Sketch:}
\begin{enumerate}
    \item By \eqref{eqn:train_posterior}, we have
    \begin{align*}
        P(\theta|D\cup z)
        = & \frac{e^{-\frac{1}{T}\ell(\theta,x_z,y_z)-\frac{1}{T}\sum_{(x,y)\in D}\ell(\theta,x,y)}}{\int e^{-\frac{1}{T}\ell(\theta,x_z,y_z)-\frac{1}{T}\sum_{(x,y)\in D}\ell(\theta,x,y)}d\theta}\\
        = & e^{-\frac{1}{T}\ell(\theta,x_z,y_z)}
        \cdot \frac{
            e^{-\frac{1}{T}\sum_{(x,y)\in D}\ell(\theta,x,y)}
            }{
                \int e^{-\frac{1}{T}\ell(\theta,x_z,y_z)-\frac{1}{T}\sum_{(x,y)\in D}\ell(\theta,x,y)}d\theta
            }
    \end{align*}
    By $\ell(\theta,x_z,y_z)\geq 0$, we further prove 
    \begin{align}
        P(\theta|D\cup z) & \geq e^{-\frac{1}{T}\ell(\theta,x_z,y_z)}\cdot \frac{e^{-\frac{1}{T}\sum_{(x,y)\in D}\ell(\theta,x,y)}}{\int e^{-\frac{1}{T}\sum_{(x,y)\in D}\ell(\theta,x,y)}d\theta}\\
        & = e^{-\frac{1}{T}\ell(\theta,x_z,y_z)}\cdot P(\theta|D)\quad (\text{By  \eqref{eqn:train_posterior}})
    \end{align}
    \item This is ensured by the convergence of the posterior distribution for trained model $\theta$ given large number of training data samples $z_1,\cdots,z_n$, as $n\rightarrow \infty$.
\end{enumerate}

We now offer details for deriving the approximated likelihood ratio test (LRT) for membership inference.

\begin{lemma}[Approximated LRT for membership inference]
    Let $(\theta, z)$ be random samples from the joint distribution of target model and target data point, specified by one of the following membership hypotheses.
    \begin{align}
        \label{eqn:LRT_hyposis_app}
        H_0:\ & D\xleftarrow{n\ i.i.d. samples}\pi(z), \theta\xleftarrow{sample}\mathcal{T}(D), z\xleftarrow{sample} \pi(z)\\
        H_1:\ & D\xleftarrow{n\ i.i.d. samples}\pi(z), \theta\xleftarrow{sample}\mathcal{T}(D), z\xleftarrow{sample} D
    \end{align}
    Then the Likelihood Ratio Test (LRT) could be approximately written as follows.
    \begin{equation}
        \label{eqn:lrt_strategy_app}
        \text{If } \ell(\theta,x_z,y_z)\leq c, \text{ reject }H_0,
    \end{equation}
    where $c$ could be an arbitrary threshold (that is constant across all $\theta$ and $z=(x_z,y_z)$). 
\end{lemma}
\begin{proof}
    
    The likelihoods function of hypothesis $H_0$ and $H_1$, given observed target model $\theta$ and target data point $z$, is as follows.
    \begin{align}
        \label{eqn:H0_joint_dist_app}
        L(H_0|\theta,z)&=P_{H_0}(\theta,z) =\pi(z)\cdot \mathbf{E}_{D\sim\pi^n} [P(\theta|D)]\\
        L(H_1|\theta,z)&=P_{H_1}(\theta,z) = \sum_{D}P_{H_1}(D,\theta,z) \\
        &= \sum_{D}\pi(z)\cdot P_{H_1}(D|z)\cdot P(\theta|D)\\
        &= \pi(z)\cdot\mathbf{E}_{D'\sim\pi^{n-1}}[P(\theta|D'\cup z)]
    \end{align}
    By using Bayes approximation of the (posterior) distribution of trained model on private dataset, we have
    \begin{equation}
        \label{eqn:train_posterior}
        P(\theta|D)\approx\frac{
            e^{-\frac{1}{T}\sum_{(x,y)\in D}\ell(\theta,x,y)}
            }{
            \int e^{-\frac{1}{T}\sum_{(x,y)\in D}\ell(\theta,x,y)} d\theta
        },
    \end{equation}
    where $T$ is a temperature constant specified by the training algorithm $\mathcal{T}$. This approximation holds for many Bayesian learning algorithms, such as stochastic gradient descent~\cite{polyak1992acceleration}. Intuitively, for $T\rightarrow 0$, \eqref{eqn:train_posterior} becomes concentrated around the optimum of the loss function, and recovers deterministic MAP (Maximum A Posteriori) inference. For $T=1$, \eqref{eqn:train_posterior} captures Bayesian posterior sampling under uniform prior and loss-based log-likelihood~\cite{welling2011bayesian}.
    Therefore, by plugging \eqref{eqn:train_posterior} and \eqref{eqn:approx_add_point} into \eqref{eqn:H0_joint_dist_app}, we prove that
    \begin{align}
        L(H_1|\theta,z) &\geq \pi(z)\cdot e^{-\frac{1}{T}\ell(\theta,z_x,z_y)}\cdot \mathbf{E}_{D'\sim\pi^{n-1}}[P(\theta|D')]\\
        \text{(By \eqref{eqn:approx_expect_model_dist}) } &\approx\pi(z)\cdot e^{-\frac{1}{T}\ell(\theta,z_x,z_y)}\cdot \mathbf{E}_{D\sim\pi^{n}}[P(\theta|D)]
    \end{align}
    Therefore the LRT statistics is 
    \begin{align}
        LR(\theta,z) & =\frac{L(H_0|\theta,z)}{L(H_1|\theta,z)}\leq e^{\frac{1}{T}\ell(\theta,z_x,z_y)} \frac{\mathbf{E}_{D\sim\pi^n} [P(\theta|D)]}{\mathbf{E}_{D'\sim\pi^{n-1}} [P(\theta|D')]} \nonumber \\
        & \approx e^{\frac{1}{T}\ell(\theta,z_x,z_y)} \label{eqn:lrt_statistics_app}
    \end{align}
    The LRT hypothesis test rejects $H_0$ when the LRT statistic is small. By \eqref{eqn:lrt_statistics_app}, the rejection region $\{(\theta,z):\lambda(\theta,z)\leq c\}$ can be approximated as follows.
    \begin{equation}
        \label{eqn:rej_region_app}
        \Big\{(\theta,z):\ell(\theta,z_x,z_y)\leq T\cdot \log c\Big\}
    \end{equation}
\end{proof}
\input{figure_scripts/purchase100_2a_loss_threshold_histogram_for_class.tex}
\section{Experimental Setup}
\label{appendix:experimental_setup}
\subsection{Details about Target Models}

All training is done using categorical cross entropy loss function. Below we list the details for training the target model (sampled without replacement).

\txtbullet Purchase 100 setup I: 4 layer MLP model with layer units = [512, 256, 128, 64], SGD optimizer algorithm, trained on 2500 data points (sampled without replacement).

\txtbullet Purchase 100 setup II, III: 4 layer MLP model with layer units = [512, 256, 128, 64], SGD optimizer algorithm, trained on 10000 data points (sampled without replacement).

\txtbullet Purchase 100 setup IV: 4 layer MLP model with layer units = [512, 256, 128, 64], SGD with additional gradient clipping with $\ell_2$ norm 2.0, trained on 10000 data points (sampled without replacement).

\txtbullet CIFAR10 Setup I: AlexNet model, Adam optimizer, trained on 2500 data points (sampled without replacement).

\txtbullet CIFAR10 Setup II: AlexNet model, Adam optimizer, trained on 5000 data points (sampled without replacement).

\txtbullet CIFAR10 Setup III: 3-block VGGNet, SGD with momentum and an L2 regularization penalty $\lambda = 0.001$, trained on 10000 data points (sampled without replacement).

\txtbullet CIFAR10 Setup IV: Wide Resnet with depth 28 and width 2, Adam optimizer, trained on 25000 data points in expectation (Poisson sampling).

\txtbullet CIFAR 100 and MNIST setup I: 2 layer CNN with filters = [32, 64] and max pooling, SGD optimizer, trained on 2500 data points (sampled without replacement).

\txtbullet CIFAR 100 and MNIST setup II: 2 layer CNN with filters = [32, 64] and max pooling, SGD optimizer, trained on 5000 data points (sampled without replacement).

\subsection{Details about training shadow, reference or distilled model and population records}

For each configuration of target model, we train shadow, reference, and distilled models on random i.i.d. sub-splits of the data (possibly overlapping). All shadow, reference or distilled models are trained using the same model structure, training algorithm and dataset size as the target model. Below we list the number of models or number of population data points we use in each setup, as well as the distillation algorithm.

\txtbullet Purchase 100 setup I, II and III: 999 shadow, reference or distilled models for attacking each target model. We sample 1000 data points from each class for launching Attack P.

\txtbullet Purchase 100 setup IV: 29 shadow, reference or distilled models for attacking each target model. We sample 1000 data points from each class for launching Attack P.

\txtbullet CIFAR10 setup I, II:  399 shadow, reference or distilled models for attacking each target model. We sample 400 data points from each class for launching Attack P.

\txtbullet CIFAR10 setup III:  29 shadow, reference or distilled models for attacking each target model. We sample 400 data points from each class for launching Attack P.

\txtbullet CIFAR10 setup IV:  15 shadow, reference or distilled models for attacking each target model. We do not luanch Attack P for this setup.

\txtbullet CIFAR 100: We sample 400 data points from each class for launching Attack P.

\txtbullet MNIST: We sample 1000 data points from each class for launching Attack P.

{\renewcommand\normalsize{\small}%
\normalsize
\begin{table*}
    \centering
    \caption{AUC Scores of all attacks on Purchase100 Dataset. Configuration I is trained on 2500 data points, configuration II is trained on 5000 data points, and configurations III and IV are trained on 10000 data points. Configuration IV has been trained with a gradient clipping norm of 2.0. Configurations I, II, and III use $n=999$ (shadow, reference, distilled) models, whereas configuration IV uses $n=29$ (shadow, reference, distilled) models.}
    \begin{tabular}{|l|l|l|l|l|l|l|}
        \hline
        & \textbf{Train Acc.} & \textbf{Test Acc.} & \textbf{Attack S}          & \textbf{Attack P}          & \textbf{Attack R}          & \textbf{Attack D}          \\
        \hline
        I & 92.36 $\pm$ 0.109 & 49.9 $\pm$ 0.056 & \textbf{0.8} $\pm$ 0.049 & \textbf{0.818} $\pm$ 0.038 & \textbf{0.825} $\pm$ 0.058
        & \textbf{0.837} $\pm$ 0.061
        \\
        \hline
        II & 99.5 $\pm$ 0.004 & 65.4 $\pm$ 0.009 & \textbf{0.754} $\pm$ 0.008 & \textbf{0.756} $\pm$ 0.006 & \textbf{0.808} $\pm$ 0.009
        & \textbf{0.831} $\pm$ 0.004
        \\
        \hline
        III & 100.0 $\pm$ 0.0 & 75.5 $\pm$ 0.004 & \textbf{0.687} $\pm$ 0.003 & \textbf{0.687} $\pm$ 0.003 & \textbf{0.755} $\pm$ 0.004
        & \textbf{0.768} $\pm$ 0.002
        \\
        \hline
        IV & 95.74 $\pm$ 0.01 & 71.71 $\pm$ 0.009 & \textbf{0.649} $\pm$ 0.004 & \textbf{0.656} $\pm$ 0.005 & \textbf{0.701} $\pm$ 0.009
        & \textbf{0.717} $\pm$ 0.006
        \\
        \hline
    \end{tabular}
    \label{tab:purchase100_auc_scores}
\end{table*}
}

{\renewcommand\normalsize{\small}%
\normalsize
\begin{table*}[ht!]
    \centering
    \caption{AUC Scores of all attacks on CIFAR10 Dataset. Configuration I is trained on 2500 data points, and configuration II is trained on 5000 data points. Configurations I and II are trained using AlexNet. Configuration III is trained on 10000 data points using a 3-block VGGNet. Here we report the results of the attacks for Configurations I and II using $n=399$ (shadow, distilled, reference) models, and $n=29$ for Configuration III. Models in setups I and II have been trained using the Adam optimizer, whereas models in setup III have been trained using the SGD optimizer with momentum and an L2 regularization penalty $\lambda = 0.001$.}
    \begin{tabular}{|l|l|l|l|l|l|l|}
        \hline
        & \textbf{Train Acc.} & \textbf{Test Acc.} & \textbf{Attack S} & \textbf{Attack P} & \textbf{Attack R} & \textbf{Attack D} \\
        \hline
        I & 96.2 $\pm$ 0.046 & 40.9 $\pm$ 0.029 & \textbf{0.870} $\pm$ 0.018 & \textbf{0.857} $\pm$ 0.023 & \textbf{0.874} $\pm$ 0.018
        & \textbf{0.871} $\pm$ 0.007
        \\
        \hline
        II & 97.8 $\pm$ 0.012 & 45.9 $\pm$ 0.010 & \textbf{0.860} $\pm$ 0.014 & \textbf{0.868} $\pm$ 0.008 & \textbf{0.858} $\pm$ 0.019
        & \textbf{0.889} $\pm$ 0.011
        \\
        \hline
        III & 97.4 $\pm$ 0.004 & 68.2 $\pm$ 0.011 & \textbf{0.707} $\pm$ 0.011 & \textbf{0.709} $\pm$ 0.01 & \textbf{0.792} $\pm$ 0.013
        & \textbf{0.806} $\pm$ 0.003
        \\
        \hline
        IV & 99.89 $\pm$ 0.007 & 77.37 $\pm$ 0.014 & - & - & \textbf{0.697} $\pm$ 0.011
        & \textbf{0.704} $\pm$ 0.009
        \\
        \hline
    \end{tabular}
    \label{tab:cifar10_auc_scores}
\end{table*}
}

{\renewcommand\normalsize{\small}%
\normalsize
\begin{table*}[ht!]
    \centering
    \caption{AUC Scores of all attacks on CIFAR100 Dataset. Configuration I is trained on 2500 data points, and Configuration II is trained on 5000 data points.}
    \begin{tabular}{|l|l|l|l|l|l|l|}
        \hline
        & \textbf{Train Acc.} & \textbf{Test Acc.} & \textbf{Attack S} & \textbf{Attack P} & \textbf{Attack R} & \textbf{Attack D} \\
        \hline
        I & 97.4 $\pm$ 0.013 & 14.9 $\pm$ 0.009 & \textbf{0.959} $\pm$ 0.005 & \textbf{0.960} $\pm$ 0.004 & \textbf{0.964} $\pm$ 0.006
        & \textbf{0.957} $\pm$ 0.0003
        \\
        \hline
        II & 97.9 $\pm$ 0.006 & 20.4 $\pm$ 0.006 & \textbf{0.944} $\pm$ 0.004 & \textbf{0.945} $\pm$ 0.003 & \textbf{0.945} $\pm$ 0.006
        & \textbf{0.936} $\pm$ 0.0
        \\
        \hline
    \end{tabular}
    \label{tab:cifar100_auc_scores}
\end{table*}
}

\subsection{Details for smoothing the loss distributions and computing continuous percentiles}

\label{ssec:smoothing}

For constructing the attacks in each setup (i.e., for computing the attack threshold), we compute $\alpha$-percentiles of smoothed loss distribution rather than the discrete histogram (over shadow models, population records, reference models or distilled models). This is to support computing attack threshold under requirement of small FPR $\alpha<\frac{1}{N+1}$, where $N+1$ is the number of samples in the loss histogram. We experiment with the following four smoothing techniques for each setup, and report the method that gives the best attack performance.

\begin{enumerate}
    \item Linear interpolation of percentile: given a discrete histogram with $N + 1$ loss values $l_0\leq \cdots \leq l_N$, then linear interpolation computes $\alpha$-percentile $p(\alpha)$ for any $0\leq \alpha\leq 1$ as follows.
    
    \begin{align*}
        p(\alpha) = & l_{\lfloor\alpha \cdot N\rfloor}\cdot (\lfloor\alpha \cdot N\rfloor + 1 - \alpha \cdot N) \\
        & + \left(\alpha \cdot N - \lfloor\alpha \cdot N\rfloor\right)\cdot l_{\lfloor\alpha \cdot N\rfloor + 1}
    \end{align*}
    \item Logit rescaling method~\cite{carlini2022membership}: given a discrete histogram with $N + 1$ loss values $l_0\leq \cdots \leq l_N$, we fit a gaussian distribution $\mathcal{N}(\mu, \sigma^2)$ on the following values
    
    \begin{align*}
        \{\log\frac{e^{-l_0}}{1 - e^{-\ell_0}}, \log\frac{e^{-l_1}}{1 - e^{-\ell_1}} \cdots, \log\frac{e^{-l_N}}{1 - e^{-\ell_N}}\},
    \end{align*}

    which are values derived by applying logit rescaling transformation on the loss values. Then, the smoothed alpha-percentile $p(\alpha)$ gives the $1-\alpha$-percentile for the above Gaussian distribution, after the logit rescaling transformation. That is
    \begin{align*}
        \log\frac{e^{- p(\alpha)}}{1- e^{-p(\alpha)}} = (1-\alpha)\text{-percentile of }\mathcal{N}(\mu, \sigma^2)
    \end{align*}

    Solution to this equation gives the percentile $p(\alpha)$ for the smoothed loss distribution.
    
    \item Minimum of threshold derived by linear interpolation method and Logit rescaling method~\cite{carlini2022membership}: given a discrete histogram with $N + 1$ loss values $l_0\leq \cdots \leq l_N$, suppose that linear interpolation method gives an estimate $p^{linear}(\alpha)$ for the $\alpha$-percentile, while logit rescaling method gives an estimate $p^{logit}(\alpha)$ for the $\alpha$-percentile. Then this method computes the $\alpha$-percentile estimate $p(\alpha)$ as the minimum of two thresholds. That is
    
    \begin{align*}
        p(\alpha) = \min\left(p^{linear}(\alpha), p^{logit}(\alpha)\right)
    \end{align*}

    This means we are being conservative about inferring a target as member, i.e., only when both smoothing methods give FPR estimate at $\alpha$, we then infer the target as member.
    \item Average of confidence derived by linear interpolation method and Logit rescaling method~\cite{carlini2022membership}: given a discrete histogram with $N + 1$ loss values $l_0\leq \cdots \leq l_N$, suppose that linear interpolation method gives a smoothed CDF $F^{linear}(\ell): \mathbb{R}\rightarrow [0, 1]$ to estimate the histogram, while logit rescaling method gives a smoothed CDF $F^{logit}(\ell): \mathbb{R}\rightarrow [0,1]$ for the histogram. Then this method computes the confidence estimate $F(\ell)$ for the target's loss value $\ell$ as the average of confidence given by two smoothing techniques. That is
    
    \begin{align*}
        F(\ell) = \frac{1}{2}\left(F^{linear}(\ell) + F^{logit}(\ell)\right)
    \end{align*}

    The inverse $F^{-1}: [0,1]\rightarrow [R]$ of this (monotonic) CDF function then serve to estimate an attack threshold for any given $\alpha\in [0,1]$. This intuitively means that we are weighting the two smoothing methods equally in terms of their decision to inferring a target as member or not.
\end{enumerate}

\section{Additional Empirical Results}
\subsection{Illustrations of Attack Thresholds on Different Targets}

\label{sec:additional_loss_hist}

In this section, we visualize the distribution of the attack thresholds for Attacks S, P, R, and D. We report the thresholds used by different attacks for Purchase100 Dataset experimental setup II over 10 randomly chosen target models (3 target models for Attack D) and their train or test data records in one class. This is to show how different are the attack thresholds in Attack S, P, R and D on different random target models and target data records, and to complement Figure~\ref{fig:purchase100_2a_all_attack_loss_dist_plots_new} (which only shows loss thresholds on four pairs of target model and target record). Observe that Attack S uses the same threshold for all the considered targets, and Attack P also uses highly similar thresholds for all the considered targets. On the contrary, Attack R and D use more diverse thresholds for adapting to different target models and target records. This is consistent with Section~\ref{ssec:summary_attack_explain} and show the different threshold dependency of Attack S, P, R and D. 
\label{appendix:additional_experiments}

In this section, we report the result of our attacks (Attack S, P, R and D) on models trained in different setups and more datasets. We show that a general trend of reduced uncertainty and increased attack performance when we go from Attack S and P to Attack R and D, across different training algorithm and datasets.

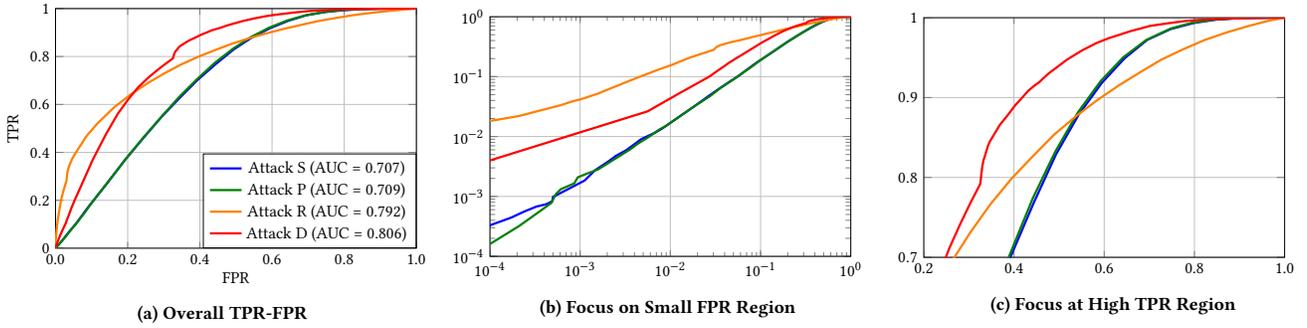
\begin{figure*}[h]
    \centering
    {\tikzset{external/export=false}
    \begin{subfigure}{.33\textwidth}
        \begin{tikzpicture}[scale=0.7]
            \begin{axis}
               [name=cifar10vggplot,
                    xlabel={FPR}, ylabel={TPR},
                    xmin = 0, xmax = 1,
                    ymin = 0, ymax = 1, yscale=0.8,
                    xtick={0.0,0.2,0.4,0.6,0.8,1.0}, xticklabels={0.0,0.2,0.4,0.6,0.8,1.0},
                    grid = major, title style={yshift=0.9cm},
                    legend style={at={(1.0, 0.0)},anchor=south east}]
                % attack S
                        \addplot[solid, very thick, blue, no marks] table[skip first n=1,x index=1, y index=2, col sep=comma] {"data/tpr_vs_fpr_curves/cifar10_vgg_ccs/attack_S_linear_itp_all_models_tpr_vs_fpr_data.csv"};
                % attack P
                        \addplot[solid, very thick, green!50!black, no marks] table[skip first n=1,x index=1, y index=2, col sep=comma] {"data/tpr_vs_fpr_curves/cifar10_vgg_ccs/attack_P_linear_itp_all_models_tpr_vs_fpr_data.csv"};
                % attack R
                        \addplot[solid, very thick, orange, no marks] table[skip first n=1,x index=1, y index=2, col sep=comma] {"data/tpr_vs_fpr_curves/cifar10_vgg_ccs/attack_R_linear_itp_all_models_tpr_vs_fpr_data.csv"};
                % attack D
                    \addplot[solid, very thick, red, no marks] table[skip first n=1,x index=1, y index=2, col sep=comma] {"data/tpr_vs_fpr_curves/cifar10_vgg_ccs/attack_D_linear_itp_all_models_tpr_vs_fpr_data.csv"};
               \legend{Attack S (AUC = 0.707), Attack P (AUC = 0.709), Attack R (AUC = 0.792), Attack D (AUC = 0.806)}
            \end{axis}
        \end{tikzpicture}
        \caption{Overall TPR-FPR}
    \end{subfigure}\hfill
    \begin{subfigure}{.33\textwidth}
        \begin{tikzpicture}[scale=0.7]
            \begin{axis}
               [name=cifar10vggplot_low_FPR,
                   ymode = log,
                   ymin = 0.0001, ymax = 1, yscale=0.8,
                   xmode = log,
                   xmin = 0.0001, xmax = 1,
                   xtick={0.0001,0.001,0.01,0.1,1.0}, xticklabels={$10^{-4}$,$10^{-3}$,$10^{-2}$,$10^{-1}$,$10^{0}$},
                   grid = major, title style={yshift= 0.9cm},
                   legend style={at={(1.0, -0.25)},anchor=south east}]
               % attack S
                       \addplot[solid, very thick, blue, no marks] table[skip first n=1,x index=1, y index=2, col sep=comma] {"data/tpr_vs_fpr_curves/cifar10_vgg_ccs/attack_S_linear_itp_all_models_tpr_vs_fpr_data.csv"};
               % attack P
                      \addplot[solid, very thick, green!50!black, no marks] table[skip first n=1,x index=1, y index=2, col sep=comma] {"data/tpr_vs_fpr_curves/cifar10_vgg_ccs/attack_P_linear_itp_all_models_tpr_vs_fpr_data.csv"};
               % attack R
                      \addplot[solid, very thick, orange, no marks] table[skip first n=1,x index=1, y index=2, col sep=comma] {"data/tpr_vs_fpr_curves/cifar10_vgg_ccs/attack_R_linear_itp_all_models_tpr_vs_fpr_data.csv"};
               % attack D
                    \addplot[solid, very thick, red, no marks] table[skip first n=1,x index=1, y index=2, col sep=comma] {"data/tpr_vs_fpr_curves/cifar10_vgg_ccs/attack_D_linear_itp_all_models_tpr_vs_fpr_data.csv"};
            \end{axis}
        \end{tikzpicture}
        \caption[Network]{Focus on Small FPR Region}
    \end{subfigure}\hfill
    \begin{subfigure}{.33\textwidth}
        \begin{tikzpicture}[scale=0.7]
            \begin{axis}
               [name=cifar10vggplot_high_TPR,
                    xmin = 0.2, xmax = 1,
                    ymin = 0.7, ymax = 1, yscale=0.8,
                    xtick={0.0,0.2,0.4,0.6,0.8,1.0}, xticklabels={0.0,0.2,0.4,0.6,0.8,1.0},
                    grid = major, title style={yshift=0.9cm},
                    legend style={at={(1.0, 0.58)},anchor=south east}]
               % attack S
                       \addplot[solid, very thick, blue, no marks] table[x=x, y=y, col sep=comma] {"data/tpr_vs_fpr_curves/cifar10_vgg_ccs/attack_S_linear_itp_all_models_tpr_vs_fpr_data.csv"};
               % attack P
                      \addplot[solid, very thick, green!50!black, no marks] table[skip first n=1,x index=1, y index=2, col sep=comma] {"data/tpr_vs_fpr_curves/cifar10_vgg_ccs/attack_P_linear_itp_all_models_tpr_vs_fpr_data.csv"};
               % attack R
                      \addplot[solid, very thick, orange, no marks] table[skip first n=1,x index=1, y index=2, col sep=comma] {"data/tpr_vs_fpr_curves/cifar10_vgg_ccs/attack_R_linear_itp_all_models_tpr_vs_fpr_data.csv"};
               % attack D
                    \addplot[solid, very thick, red, no marks] table[skip first n=1,x index=1, y index=2, col sep=comma] {"data/tpr_vs_fpr_curves/cifar10_vgg_ccs/attack_D_linear_itp_all_models_tpr_vs_fpr_data.csv"};
            \end{axis}
        \end{tikzpicture}
        \caption{Focus at High TPR Region}
    \end{subfigure}
    }
    \caption{FPR vs TPR with AUC scores for all attacks on CIFAR10 Dataset experimental setup III (details in Appendix~\ref{appendix:experimental_setup}). Attack S, R and D use 29 shadow, reference and distilled models respectively, and Attack P uses 400 population data points (per-class). The evaluated TPR and FPR are averaged over member and non-member target data of 10 target models for Attack S, P and R. For Attack D we only average TPR and FPR over 3 target models due to limited computational resources. Attack D has the highest performance in terms of AUC, followed closely by Attack R. Because there are few samples in the loss histogram constructed by Attack P (400), R and D (29 each), to obtain attack threshold for small $\alpha<0.001$, we use linear interpolation smoothing method for the loss histograms (detailed descriptions are in Section~\ref{ssec:smoothing}).}
    \label{fig:cifar10_vgg_fpr_vs_tpr_plots}
\end{figure*}

\input{figure_scripts/cifar10_resnet_with_related_work_fpr_vs_tpr_plots_ccs.tex}

\subsection{Additional Empirical Results for Evaluating Attack Performance}

\paragraph{Additional Attack Performance Results for Purchase100}

We report attack results on different training configurations on the Purchase100 dataset. Table~\ref{tab:purchase100_auc_scores} shows a general overview of the AUC scores of different attacks on different setups. We vary the size of the training dataset among different setups, to show that although privacy risk decreases as the training dataset size increases, the general trend of increased attack performance remains, when we compare Attack S, P, R and D.

\paragraph{Additional Attack Performance Results for CIFAR10}

We then turn to evaluate attacks on the CIFAR10 dataset. We report the overall AUC scores of attacks under different CIFAR10 setups in Table~\ref{tab:cifar10_auc_scores}. We also provide detailed TPR-FPR curves the CIFAR experiment setup III in Figure~\ref{fig:cifar10_vgg_fpr_vs_tpr_plots}. The overall trend of performance (AUC score and TPR-FPR tradeoff) of different Attacks remains similar to that on Purchase100 dataset, although the performances are lower.

\paragraph{Additional Attack Performance Results for  CIFAR100 and MNIST}

We evaluate the performance of attacks on more datasets: CIFAR100 and MNIST. We report the AUC scores of attacks under different setups in Table~\ref{tab:cifar100_auc_scores} and Table~\ref{tab:mnist_auc_scores}. Observe that the attack performance gaps are much smaller than Purchase100 and CIFAR10. (That is, all the attacks are successful on CIFAR100, while all the attacks perform poorly on MNIST.) However, we still observe that Attack R and D to slightly outperform Attack S and P in 3 of the 4 setups.

{\renewcommand\normalsize{\small}%
\normalsize
\begin{table*}[ht!]
    \centering
    \caption{AUC Scores of all attacks on MNIST Dataset. Configuration I is trained on 2500 data points, and Configuration II is trained on 5000 data points.}
    \begin{tabular}{|l|l|l|l|l|l|l|}
        \hline
        & \textbf{Train Acc.} & \textbf{Test Acc.} & \textbf{Attack S} & \textbf{Attack P} & \textbf{Attack R} & \textbf{Attack D} \\
        \hline
        I & 97.9 $\pm$ 0.004 & 95.8 $\pm$ 0.007 & \textbf{0.50} $\pm$ 0.004  & \textbf{0.50} $\pm$ 0.005  & \textbf{0.557} $\pm$ 0.009
        & \textbf{0.549} $\pm$ 0.006
        \\
        \hline
        II & 98.6 $\pm$ 0.001 & 97.1 $\pm$ 0.002 & \textbf{0.496} $\pm$ 0.005 & \textbf{0.496} $\pm$ 0.006 & \textbf{0.551} $\pm$ 0.011
        & \textbf{0.544} $\pm$ 0.004
        \\
        \hline
    \end{tabular}
    \label{tab:mnist_auc_scores}
\end{table*}
}

\begin{figure*}[h!]
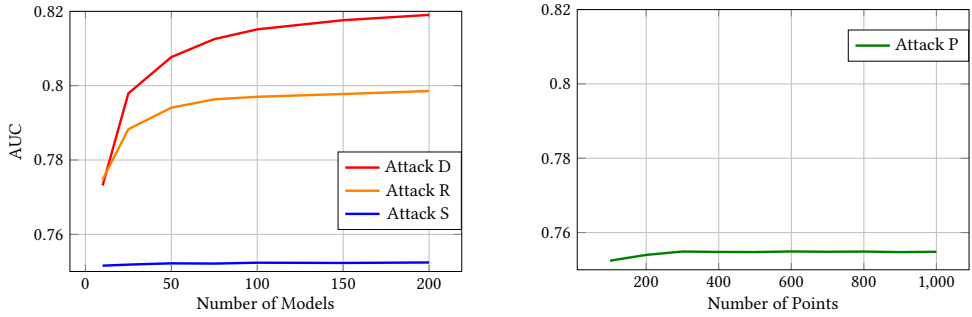

    \begin{subfigure}[b]{.45\textwidth}
        \raggedleft
    \if\compileFigures1
    \begin{tikzpicture}[scale=0.76]
    \begin{axis}[
       xlabel={Number of Models}, ylabel={AUC},
       ymin = 0.75, ymax = 0.82, yscale=0.8,
       grid = major,
       legend style={at={(1.0, 0.2)},anchor=south east}]
       % attack D
			\addplot[solid, very thick, red, no marks] table[skip first n=1,x index=1, y index=2, col sep=comma] {"data/effect_of_hyperparams/purchase100_2a/attack_D_auc_vs_num_models_data.csv"};
       % attack R
          	\addplot[solid, very thick, orange, no marks] table[skip first n=1,x index=1, y index=2, col sep=comma] {"data/effect_of_hyperparams/purchase100_2a/attack_R_auc_vs_num_models_data.csv"};
       % attack S
       		\addplot[solid, very thick, blue, no marks] table[skip first n=1,x index=1, y index=2, col sep=comma] {"data/effect_of_hyperparams/purchase100_2a/attack_S_auc_vs_num_models_data.csv"};
       \legend{Attack D, Attack R, Attack S}
    \end{axis}
    \end{tikzpicture}
    \else
        \includegraphics[]{fig/\filename-figure\thefigureNumber.pdf}
        \stepcounter{figureNumber}
    \fi
    \end{subfigure}\hfill
    \begin{subfigure}[b]{.45\textwidth}
    \if\compileFigures1
    \begin{tikzpicture}[scale=0.76]
    \begin{axis}[
       xlabel={Number of Points}, 
       ymin = 0.75, ymax = 0.82, yscale=0.8,
       grid = major,
       legend style={at={(1.0, 1.0)},anchor=south east}]
       % attack P
            \addplot[solid, very thick, green!50!black, no marks] table[skip first n=1,x index=1, y index=2, col sep=comma] {"data/effect_of_hyperparams/purchase100_2a/attack_P_auc_vs_num_data_in_class_data.csv"};
       \legend{Attack P}
    \end{axis}
    \end{tikzpicture}
    \else
        \includegraphics[]{fig/\filename-figure\thefigureNumber.pdf}
        \stepcounter{figureNumber}
    \fi
    \end{subfigure}\hspace*{\fill}
    % (a) Effect of Number of Models & (b) Effect of Number of Points
    \caption{Effect of computation cost hyperparameters on the attack AUC Scores in Purchase100 II configuration. For Attacks S, R, and D, we vary the number of shadow, reference, and distilled models from 10 to 200. For Attack P we vary the number of population data points (per class) from 100 to 1000. We observe that as more reference or distilled models are used, the AUC scores for Attacks R and D increase rapidly and then stabilize, However, the AUC score for Attack S only increases slightly with the number of shadow models. For Attack P, we similarly observe a slight increase in its AUC score with an increase in number of population data points.}
    \label{fig:purchase100_2a_effect_of_hyperparameters_new}
\end{figure*}

\subsection{Additional Experiment Results for Worst-case Vulnerable Record}

\input{figure_scripts/purchase100_2a_worst_case_neighbouring_loss_histogram_point2.tex}

In Figure~\ref{fig:purchase100_worst_case_neighbouring_loss_histogram} in Section~\ref{ssec:why_vulnerable}, we observe that the latent neighbors of a worst-case vulnerable data record may be highly correlated with the vulnerability of this record. For completeness, in this section, we offer more experiment results for another randomly chosen worst-case vulnerable (RD Correct) record in Figure~\ref{fig:purchase100_worst_case_neighbouring_loss_histogram_point2}.

\subsection{Details about comparison with concurrent work~\cite{carlini2022membership}}

We perform the concurrent attacks in~\citet{carlini2022membership} on our CIFAR10 IV setup, and provide detailed comparison between our attacks and three variants of the concurrent attacks~\cite{carlini2022membership} in Figure~\ref{fig:cifar10_resnet_with_related_work_fpr_vs_tpr_plots}. We observe under 29 shadow models (for constructing the attacks~\cite{carlini2022membership}), 29 reference models (for constructing Attack R), and 29 distilled models (for constructing Attack D), our Attack R and D has higher performance than all three variants of concurrent attack~\cite{carlini2022membership} in terms of AUC score. Attack R also achieves similar TPR at small FPR when compared to the concurrent attacks~\cite{carlini2022membership} in Figure~\ref{fig:cifar10_resnet_with_related_work_fpr_vs_tpr_plots} (b). Moreover, Attack D achieves smaller FPR at fixed high TPR in Figure~\ref{fig:cifar10_resnet_with_related_work_fpr_vs_tpr_plots} (c). Note that differential privacy aim to (roughly speaking) bound the total error (FPR and FNR: 1-TPR). So, both high TPR and low FPR are important for privacy risk auditing. 

\subsection{The effect of Computation Cost on Attack Performance}

Comparing and optimizing the computation cost of different attacks is an important problem. Different attack strategies for privacy auditing have different computational costs. In Appendix A.3., We investigate the relationship between attack computation cost (as measured by hyperparameters such as number of shadow/distilled models trained on and number of population data records used) and performance. How many models do Attacks S, R, and D need to achieve a good attack performance? How many data points does Attack P need for each class in the dataset to perform well? We plot the effect on AUC scores on varying these attack hyperparameters in Figure~\ref{fig:purchase100_2a_effect_of_hyperparameters_new}. We observe that the attack performance improves significantly with an increase in reference or distilled models for Attacks R and D respectively. For Attack S we observe a slight increase in its AUC score with an increase in the number of shadow models. For Attack P, we similarly observe a slight increase in its AUC score with an increase in number of data points per class.

\section{Discussion about Attack Performance under Defense Algorithms}

The effect of defense algorithms on the performance of attacks is an important problem.  As a small example, we experimented on gradient clipping as a heuristic defense in Purchase setup IV. In Table~\ref{tab:purchase100_auc_scores}, we observe that although the performance of our attacks is lower in general after gradient clipping (e.g. when compared to Figure~\ref{fig:purchase100_fpr_vs_tpr_plots_new} without gradient clipping), the attacks (S, P, R and D) derived from our framework still exhibit increasing strength. This trend of increasing attack strength (for Attack S, P, R and D) shows that our methodology for enhancing membership inference attacks is still effective under simple defense such as gradient clipping. We believe this is because we do not make any assumptions about the training algorithms or defense algorithms during derivation of the increasingly strong attacks in Section~\ref{section3}.

\end{document}